\RequirePackage{fix-cm}
\documentclass[smallextended]{svjour3}       
\smartqed  
\usepackage{graphicx}
\usepackage{epstopdf}
\usepackage{amsfonts}
\usepackage{epsfig}
\usepackage{graphics}
\usepackage{amsmath}
\usepackage{amssymb}
\usepackage{algorithmic}
\usepackage{algorithm}
\usepackage{subfigure}  

\begin{document}

\title{A Convergent Algorithm for Bi-orthogonal Nonnegative Matrix Tri-Factorization}


\author{Andri Mirzal}



\date{Received: date / Accepted: date}
\maketitle

\begin{abstract}
A convergent algorithm for nonnegative matrix factorization with orthogonality constraints imposed on both factors is proposed in this paper. This factorization concept was first introduced by Ding et al.~\cite{Ding1} with intent to further improve clustering capability of NMF. However, as the original algorithm was developed based on multiplicative update rules, the convergence of the algorithm cannot be guaranteed. In this paper, we utilize the technique  presented in our previous work \cite{Mirzal2} to develop the algorithm and prove that it converges to a stationary point inside the solution space.
\end{abstract}

\section{Introduction}\label{intro}

The nonnegative matrix factorization (NMF) was initially introduced with an intention to produce a meaningful decomposition for learning the parts of objects and semantic features of
articles from a document corpus \cite{Lee}. In the basic form, the NMF seeks to decompose a nonnegative matrix into a pair of other nonnegative matrices with lower ranks:
\begin{equation*}
\mathbf{A} \approx \mathbf{B}\mathbf{C},
\end{equation*}
where $\mathbf{A}\in\mathbb{R}_{+}^{M\times N}=\left[\mathbf{a}_1,\ldots,\mathbf{a}_N\right]$ denotes the data matrix, $\mathbf{B}\in\mathbb{R}_{+}^{M\times R}=\left[\mathbf{b}_1,\ldots,\mathbf{b}_R\right]$ denotes the basis matrix, $\mathbf{C}\in\mathbb{R}_{+}^{R\times N}=\left[\mathbf{c}_1,\ldots,\mathbf{c}_N\right]$ denotes the coef\mbox{}ficient matrix, $R$ denotes the number of factors which usually is chosen so that $R\ll\min(M,N)$, and $\mathbb{R}_{+}^{M\times N}$ denotes $M$ by $N$ nonnegative real matrix. The following minimization problem is usually solved to compute $\mathbf{B}$ and $\mathbf{C}$:
\begin{equation}
\min_{\mathbf{B},\mathbf{C}}J(\mathbf{B},\mathbf{C})=\frac{1}{2}\|\mathbf{A}-\mathbf{B}\mathbf{C}\|_{F}^{2}\;\,\mathrm{s.t.}\;\, \mathbf{B}\succeq\mathbf{0},\mathbf{C}\succeq\mathbf{0},
\label{eq4}
\end{equation} where the symbol $\succeq$ denotes entrywise comparison.

In order to improve clustering capability of NMF, Ding et al.~\cite{Ding1} introduced two types of orthogonal NMFs: uni-orthogonal NMF (UNMF) and bi-orthogonal nonnegative matrix tri-factorization (BNMtF) where the former imposes orthogonality constraint on either columns of $\mathbf{B}$ or rows of $\mathbf{C}$, and the latter imposes orthogonality constraints on both columns of $\mathbf{B}$ and rows of $\mathbf{C}$ simultaneously. And due to the tight constraints in the latter, they introduced the third factor to absorb the approximation error. They then proposed an algorithm for each orthogonal NMF based on the multiplicative update rules which are known to only have the nonincreasing property \cite{Lee2}. A convergent algorithm for UNMF has been presented in our previous work \cite{Mirzal2}. In this paper, we extend it to the BNMtF case. As orthogonality constraint cannot be recast into alternating nonnegativity-constrained least square (ANLS) framework (see \cite{HKim2,HKim} for discussion on ANLS) some convergent algorithms for the standard NMF, e.g., \cite{HKim,CJLin2,DKim,DKim2,JKim2,CJLin} cannot be extended to the problem.

\section{Bi-orthogonal nonnegative matrix tri-factorization}\label{nmtf}

BNMtF puts orthogonality constraints on both columns of $\mathbf{B}$ and rows of $\mathbf{C}$ and it is expected that this approach can be used to simultaneously cluster columns and rows of $\mathbf{A}$ (biclustering). The following describes the original BNMtF objective function proposed by Ding et al.~\cite{Ding1}.
\begin{align}
&\min_{\mathbf{B},\mathbf{C},\mathbf{S}}J(\mathbf{B},\mathbf{C},\mathbf{S})=\frac{1}{2}\|\mathbf{A}-\mathbf{BSC}\|_{F}^{2} \label{eq37}\\
&\mathrm{s.t.}\;\,\mathbf{B}\succeq\mathbf{0},\;\mathbf{S}\succeq\mathbf{0},\;\mathbf{C}\succeq\mathbf{0},\;\frac{1}{2}\big(\mathbf{CC}^T-\mathbf{I}\big)=\mathbf{0},\;\mathrm{and}\;\frac{1}{2}\big(\mathbf{B}^T\mathbf{B}-\mathbf{I}\big)=\mathbf{0} \nonumber
\end{align}
where $\mathbf{B}\in\mathbb{R}_{+}^{M\times P}$ and $\mathbf{C}\in\mathbb{R}_{+}^{Q\times N}$; and $\mathbf{S}\in\mathbb{R}_{+}^{P\times Q}$ is introduced to absorb the scale differences of $\mathbf{A}$ and $\mathbf{B}\mathbf{C}$ due to the strict orthogonality constraints on $\mathbf{B}$ and $\mathbf{C}$. We will set $P=Q$ for the rest of this paper (for biclustering task, it is natural to have the same number of clusters for both columns and rows of the data matrix).

The Karush-Kuhn-Tucker (KKT) function of the objective function can be defined as:
\begin{align*}
L(\mathbf{B},\mathbf{C},\mathbf{S})=\;&J(\mathbf{B},\mathbf{C},\mathbf{S})-\mathrm{tr}\;\big(\mathbf{\Gamma}_{\mathbf{B}}\mathbf{B}^T\big)-\mathrm{tr}\;\big(\mathbf{\Gamma}_{\mathbf{S}}\mathbf{S}^T\big)-\mathrm{tr}\;\big(\mathbf{\Gamma}_{\mathbf{C}}\mathbf{C}\big) \\
&+ \frac{1}{2}\mathrm{tr}\;\big(\mathbf{\Lambda}_{\mathbf{C}}\big(\mathbf{CC}^T-\mathbf{I}\big)\big) + \frac{1}{2}\mathrm{tr}\;\big(\mathbf{\Lambda}_{\mathbf{B}}\big(\mathbf{B}^T\mathbf{B}-\mathbf{I}\big)\big),
\end{align*}
where $\mathbf{\Gamma}_{\mathbf{B}}\in\mathbb{R}_{+}^{M\times P}$, $\mathbf{\Gamma}_{\mathbf{S}}\in\mathbb{R}_{+}^{P\times Q}$, $\mathbf{\Gamma}_{\mathbf{C}}\in\mathbb{R}_{+}^{N\times Q}$, $\mathbf{\Lambda}_{\mathbf{C}}\in\mathbb{R}_{+}^{Q\times Q}$,  and $\mathbf{\Lambda}_{\mathbf{B}}\in\mathbb{R}_{+}^{P\times P}$ are the KKT multipliers. 

An equivalent objective function to eq.~\ref{eq37} was proposed by Ding et al.~\cite{Ding1} to absorb the orthogonality constraints into the objective:
\begin{align}
\min_{\mathbf{B},\mathbf{C},\mathbf{S}}J(\mathbf{B},\mathbf{C},\mathbf{S})=\;&\frac{1}{2}\|\mathbf{A}-\mathbf{BSC}\|_{F}^{2} + \frac{1}{2}\mathrm{tr}\;\big(\mathbf{\Lambda}_{\mathbf{C}}\big(\mathbf{C}\mathbf{C}^T-\mathbf{I}\big)\big) + \nonumber\\
\; &\frac{1}{2}\mathrm{tr}\;\big(\mathbf{\Lambda}_{\mathbf{B}}\big(\mathbf{B}^T\mathbf{B}-\mathbf{I}\big)\big)\label{eq39}.
\end{align}

The KKT conditions for objective in eq.~\ref{eq39} can be written as:
\begin{equation*}
\begin{array}{rrr}
\mathbf{B}^*\succeq\mathbf{0}, & \mathbf{S}^*\succeq\mathbf{0}, & \mathbf{C}^*\succeq\mathbf{0}, \\
\nabla_{\mathbf{B}}J(\mathbf{B}^*)=\mathbf{\Gamma}_{\mathbf{B}}\succeq\mathbf{0}, & \nabla_{\mathbf{S}}J(\mathbf{S}^*)=\mathbf{\Gamma}_{\mathbf{S}}\succeq\mathbf{0}, & \nabla_{\mathbf{C}}J(\mathbf{C}^*)=\mathbf{\Gamma}_{\mathbf{C}}^T\succeq\mathbf{0},\\
\nabla_{\mathbf{B}}J(\mathbf{B}^*)\odot\mathbf{B}^*=\mathbf{0}, & \nabla_{\mathbf{S}}J(\mathbf{S}^*)\odot\mathbf{S}^*=\mathbf{0}, & \nabla_{\mathbf{C}}J(\mathbf{C}^*)\odot\mathbf{C}^*=\mathbf{0}, \nonumber
\end{array}
\end{equation*} where $\odot$ denotes entrywise multiplication operation, and
\begin{align*}
\nabla_{\mathbf{B}}J(\mathbf{B})&=\mathbf{BSCC}^T\mathbf{S}^T-\mathbf{AC}^T\mathbf{S}^T+\mathbf{B\Lambda}_{\mathbf{B}}, \\
\nabla_{\mathbf{C}}J(\mathbf{C})&=\mathbf{S}^T\mathbf{B}^T\mathbf{BSC}-\mathbf{S}^T\mathbf{B}^T\mathbf{A}+\mathbf{\Lambda}_{\mathbf{C}}\mathbf{C}, \\
\nabla_{\mathbf{S}}J(\mathbf{S})&=\mathbf{B}^T\mathbf{BSCC}^T-\mathbf{B}^T\mathbf{AC}^T.
\end{align*} Then, by using the multiplicative updates \cite{Lee2}, Ding et al.~\cite{Ding1} derived BNMtF algorithm as follows:
\begin{align}
b_{mp} &\longleftarrow b_{mp}\frac{\big(\mathbf{A}\mathbf{C}^T\mathbf{S}^T\big)_{mp}}{\left[\mathbf{B}\big(\mathbf{SCC}^T\mathbf{S}^T + \mathbf{\Lambda}_{\mathbf{B}}\big)\right]_{mp}}, \label{eq44}
\\
c_{qn} &\longleftarrow c_{qn}\frac{\big(\mathbf{S}^T\mathbf{B}^T\mathbf{A}\big)_{qn}}{\left[\big(\mathbf{S}^T\mathbf{B}^T\mathbf{B}\mathbf{S} + \mathbf{\Lambda}_{\mathbf{C}}\big)\mathbf{C}\right]_{qn}}, \label{eq45} 
\\
s_{pq} &\longleftarrow s_{pq}\frac{(\mathbf{B}^T\mathbf{A}\mathbf{C}^T)_{pq}}{(\mathbf{B}^T\mathbf{BSCC}^T)_{pq}}, \label{eq46}
\end{align} with
\begin{align*}
\mathbf{\Lambda}_{\mathbf{B}} &= \mathbf{B}^T\mathbf{AC}^T\mathbf{S}^T - \mathbf{SCC}^T\mathbf{S}^T\;\;\;\; \mathrm{and}\\
\mathbf{\Lambda}_{\mathbf{C}} &= \mathbf{S}^T\mathbf{B}^T\mathbf{AC}^T - \mathbf{S}^T\mathbf{B}^T\mathbf{B}\mathbf{S}
\end{align*}
are derived exactly for the diagonal entries, and \emph{approximately} for off-diagonal entries by relaxing the nonnegativity constraints. 

The complete BNMtF algorithm proposed in \cite{Ding1} is shown in algorithm \ref{algorithm3} where $\delta$ denotes some small positive number (note that the normalization step is not recommended as it will change the objective value). As there are approximations in deriving $\mathbf{\Lambda}_{\mathbf{B}}$ and $\mathbf{\Lambda}_{\mathbf{C}}$, algorithm \ref{algorithm3} may or may not be minimizing the objective eq.~\ref{eq39}. Further, the auxiliary function used by the authors to prove the nonincreasing property is for the algorithm in eqs.~\ref{eq44} -- \ref{eq46}, not for algorithm \ref{algorithm3}. So there is no guarantee that algorithm \ref{algorithm3} has the nonincreasing property. Figure \ref{fig2} shows error per iteration of algorithm \ref{algorithm3} in Reuters4 dataset (see section \ref{results} for detailed info about the dataset). As the algorithm \ref{algorithm3} not only does not have the nonincreasing property but also fails to minimize the objective function, it is clear that the assumptions used to obtain $\mathbf{\Lambda}_{\mathbf{B}}$ and $\mathbf{\Lambda}_{\mathbf{C}}$ are not acceptable.

\begin{algorithm}
\caption{: Original BNMtF algorithm.}
\label{algorithm3}
\begin{algorithmic}
\STATE Initialize $\mathbf{B}^{(0)}$, $\mathbf{C}^{(0)}$, and $\mathbf{S}^{(0)}$ with positive matrices to avoid zero locking.
\FOR {$k=0,\ldots,K$}
\STATE \begin{align*} b_{mp}^{(k+1)} &\longleftarrow b_{mp}^{(k)}\frac{\big(\mathbf{AC}^{(k)T}\mathbf{S}^{(k)T}\big)_{mp}}{\big(\mathbf{B}^{(k)}\mathbf{B}^{(k)T}\mathbf{A}\mathbf{C}^{(k)T}\mathbf{S}^{(k)T}\big)_{mp}+\delta}\;\;\forall m,p \\
c_{qn}^{(k+1)} &\longleftarrow c_{qn}^{(k)}\frac{\big(\mathbf{S}^{(k)T}\mathbf{B}^{(k+1)T}\mathbf{A}\big)_{qn}}{\big(\mathbf{S}^{(k)T}\mathbf{B}^{(k+1)T}\mathbf{A}\mathbf{C}^{(k)T}\mathbf{C}^{(k)}\big)_{qn}+\delta}\;\;\forall q,n \\
s_{pq}^{(k+1)} &\longleftarrow s_{pq}^{(k)}\frac{\big(\mathbf{B}^{(k+1)T}\mathbf{A}\mathbf{C}^{(k+1)T}\big)_{pq}}{\big(\mathbf{B}^{(k+1)T}\mathbf{B}^{(k+1)}\mathbf{S}^{(k)}\mathbf{C}^{(k+1)}\mathbf{C}^{(k+1)T}\big)_{pq}+\delta}\;\;\forall p,q \end{align*}
\ENDFOR
\end{algorithmic}
\end{algorithm}

\begin{figure}
\begin{center}
\includegraphics[width=0.6\textwidth]{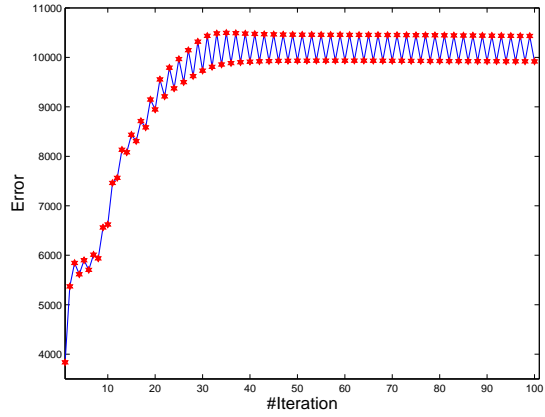}
\caption{Error per iteration of algorithm \ref{algorithm3} for Reuters4 dataset.}
\label{fig2}
\end{center}
\end{figure}

\section{A Convergent Algorithm for BNMtF}\label{convergent}

We define BNMtF problem with following equation:
\begin{align}
&\min_{\mathbf{B},\mathbf{C},\mathbf{S}}J(\mathbf{B},\mathbf{C},\mathbf{S})=\frac{1}{2}\|\mathbf{A}-\mathbf{B}\mathbf{S}\mathbf{C}\|_{F}^{2} + \frac{\alpha}{2}\|\mathbf{CC}^T-\mathbf{I}\|_{F}^{2} + \frac{\beta}{2}\|\mathbf{B}^T\mathbf{B}-\mathbf{I}\|_{F}^{2} \label{eq90}\\
&\mathrm{s.t.}\;\, \mathbf{B}\succeq\mathbf{0},\mathbf{C}\succeq\mathbf{0},\mathbf{S}\succeq\mathbf{0}, \nonumber
\end{align}
with $\alpha$ and $\beta$ are regularization parameters to adjust the degree of orthogonality of $\mathbf{C}$ and $\mathbf{B}$ respectively. The KKT function of the objective can be written as:
\begin{align*}
L(\mathbf{B},\mathbf{C},\mathbf{S})=\;&J(\mathbf{B},\mathbf{C},\mathbf{S})-\mathrm{tr}\;\big(\mathbf{\Gamma}_{\mathbf{B}}\mathbf{B}^T\big)-\mathrm{tr}\;\big(\mathbf{\Gamma}_{\mathbf{S}}\mathbf{S}^T\big)-\mathrm{tr}\;\big(\mathbf{\Gamma}_{\mathbf{C}}\mathbf{C}\big). 
\end{align*}
And the KKT conditions are:
\begin{equation}
\begin{array}{rrr}
\mathbf{B}^*\succeq\mathbf{0}, & \mathbf{S}^*\succeq\mathbf{0}, & \mathbf{C}^*\succeq\mathbf{0}, \\
\nabla_{\mathbf{B}}J(\mathbf{B}^*)=\mathbf{\Gamma}_{\mathbf{B}}\succeq\mathbf{0}, & \nabla_{\mathbf{S}}J(\mathbf{S}^*)=\mathbf{\Gamma}_{\mathbf{S}}\succeq\mathbf{0}, & \nabla_{\mathbf{C}}J(\mathbf{C}^*)=\mathbf{\Gamma}_{\mathbf{C}}^T\succeq\mathbf{0},\\
\nabla_{\mathbf{B}}J(\mathbf{B}^*)\odot\mathbf{B}^*=\mathbf{0}, & \nabla_{\mathbf{S}}J(\mathbf{S}^*)\odot\mathbf{S}^*=\mathbf{0}, & \nabla_{\mathbf{C}}J(\mathbf{C}^*)\odot\mathbf{C}^*=\mathbf{0}, \label{eq91}
\end{array}
\end{equation}
where
\begin{align*}
\nabla_{\mathbf{B}}J(\mathbf{B})&=\mathbf{BSCC}^T\mathbf{S}^T-\mathbf{AC}^T\mathbf{S}^T+\beta\mathbf{BB}^T\mathbf{B}-\beta\mathbf{B}, \\
\nabla_{\mathbf{C}}J(\mathbf{C})&=\mathbf{S}^T\mathbf{B}^T\mathbf{BSC}-\mathbf{S}^T\mathbf{B}^T\mathbf{A}+\alpha\mathbf{CC}^T\mathbf{C}-\alpha\mathbf{C}, \\
\nabla_{\mathbf{S}}J(\mathbf{S})&=\mathbf{B}^T\mathbf{BSCC}^T-\mathbf{B}^T\mathbf{AC}^T.
\end{align*}

As shown by Lee \& Seung \cite{Lee2}, a multiplicative update rules (MUR) based algorithm can be derived by utilizing complementary slackness in the KKT conditions (the last line in equations \ref{eq91}). Therefore, a MUR based algorithm for our BNMtF problem can be written as:
\begin{align*} 
b_{mp} &\longleftarrow b_{mp}\frac{\big(\mathbf{AC}^T\mathbf{S}^T+\beta\mathbf{B}\big)_{mp}}{\big(\mathbf{BSCC}^T\mathbf{S}^T+\beta\mathbf{BB}^T\mathbf{B}\big)_{mp}}, \\
c_{qn} &\longleftarrow c_{qn}\frac{\big(\mathbf{S}^T\mathbf{B}^T\mathbf{A}+\alpha\mathbf{C}\big)_{qn}}{\big(\mathbf{S}^T\mathbf{B}^T\mathbf{BSC}+\alpha\mathbf{CC}^T\mathbf{C}\big)_{qn}}, \\
s_{pq} &\longleftarrow s_{pq}\frac{\big(\mathbf{B}^T\mathbf{AC}^T\big)_{pq}}{\big(\mathbf{B}^T\mathbf{BSCC}^T\big)_{pq}}.
\end{align*} The complete MUR algorithm is given in algorithm \ref{algorithm7}, and the additive update rules (AUR) version is given in algorithm \ref{algorithm8} (please see e.g., \cite{CJLin2,Mirzal2} for detailed discussion about AUR based NMF algorithms). As shown, the AUR algorithm can be initialized using nonnegative matrices as it does not inherit zero locking problems from its MUR algorithm counterpart.

\begin{algorithm}
\caption{: The MUR based algorithm for BNMtF problem in eq.~\ref{eq90}.}
\label{algorithm7}
\begin{algorithmic}
\STATE Initialize $\mathbf{B}^{(0)}$, $\mathbf{C}^{(0)}$, and $\mathbf{S}^{(0)}$ with positive matrices to avoid zero locking.
\FOR {$k=0,\ldots,K$}
\STATE \begin{align*} b_{mp}^{(k+1)} &\longleftarrow b_{mp}^{(k)}\frac{\big(\mathbf{AC}^{(k)T}\mathbf{S}^{(k)T} + \beta\mathbf{B}^{(k)}\big)_{mp}}{\big(\mathbf{B}^{(k)}\mathbf{S}^{(k)}\mathbf{C}^{(k)}\mathbf{C}^{(k)T}\mathbf{S}^{(k)T}+\beta\mathbf{B}^{(k)}\mathbf{B}^{(k)T}\mathbf{B}^{(k)}\big)_{mp}+\delta}\;\;\forall m,p 
\\
c_{qn}^{(k+1)} &\longleftarrow c_{qn}^{(k)}\frac{\big(\mathbf{S}^{(k)T}\mathbf{B}^{(k+1)T}\mathbf{A}+\alpha\mathbf{C}^{(k)}\big)_{qn}}{\big(\mathbf{S}^{(k)T}\mathbf{B}^{(k+1)T}\mathbf{B}^{(k+1)}\mathbf{S}^{(k)}\mathbf{C}^{(k)}+\alpha\mathbf{C}^{(k)}\mathbf{C}^{(k)T}\mathbf{C}^{(k)}\big)_{qn}+\delta}\;\;\forall q,n 
\\
s_{pq}^{(k+1)} &\longleftarrow s_{pq}^{(k)}\frac{\big(\mathbf{B}^{(k+1)T}\mathbf{A}\mathbf{C}^{(k+1)T}\big)_{pq}}{\big(\mathbf{B}^{(k+1)T}\mathbf{B}^{(k+1)}\mathbf{S}^{(k)}\mathbf{C}^{(k+1)}\mathbf{C}^{(k+1)T}\big)_{pq}+\delta}\;\;\forall p,q 
\end{align*}
\ENDFOR
\end{algorithmic}
\end{algorithm}

\begin{algorithm}
\caption{: The AUR based algorithm for BNMtF problem in eq.~\ref{eq90}.}
\label{algorithm8}
\begin{algorithmic}
\STATE Initialize $\mathbf{B}^{(0)}$, $\mathbf{C}^{(0)}$, and $\mathbf{S}^{(0)}$ with nonnegative matrices.
\FOR {$k=0,\ldots,K$}
\STATE \begin{align} b_{mp}^{(k+1)} \longleftarrow & \;b_{mp}^{(k)} - \frac{\bar{b}_{mp}^{(k)}\times\nabla_{\mathbf{B}}J(\mathbf{B}^{(k)},\mathbf{S}^{(k)},\mathbf{C}^{(k)})_{mp}}{\big(\mathbf{\bar{B}}^{(k)}\mathbf{S}^{(k)}\mathbf{C}^{(k)}\mathbf{C}^{(k)T}\mathbf{S}^{(k)T}+\beta\mathbf{\bar{B}}^{(k)}\mathbf{\bar{B}}^{(k)T}\mathbf{\bar{B}}^{(k)}\big)_{mp}+\delta_{\mathbf{B}}^{(k)}}\;\;\forall m,p \label{eq98}\\
c_{qn}^{(k+1)} \longleftarrow & \;c_{qn}^{(k)} - \frac{\bar{c}_{qn}^{(k)}\times\nabla_{\mathbf{C}}J(\mathbf{B}^{(k+1)},\mathbf{S}^{(k)},\mathbf{C}^{(k)})_{qn}}{\big(\mathbf{S}^{(k)T}\mathbf{B}^{(k+1)T}\mathbf{B}^{(k+1)}\mathbf{S}^{(k)}\mathbf{\bar{C}}^{(k)}+\alpha\mathbf{\bar{C}}^{(k)}\mathbf{\bar{C}}^{(k)T}\mathbf{\bar{C}}^{(k)}\big)_{qn}+\delta_{\mathbf{C}}^{(k)}} \;\;\forall q,n \label{eq99} \\
s_{pq}^{(k+1)} \longleftarrow & \;s_{pq}^{(k)} - \frac{\bar{s}_{pq}^{(k)}\times\nabla_{\mathbf{S}}J(\mathbf{B}^{(k+1)},\mathbf{S}^{(k)},\mathbf{C}^{(k+1)})_{pq}}{\big(\mathbf{B}^{(k+1)T}\mathbf{B}^{(k+1)}\mathbf{\bar{S}}^{(k)}\mathbf{C}^{(k+1)}\mathbf{C}^{(k+1)T}\big)_{pq}+\delta_{\mathbf{S}}^{(k)}} \;\;\forall p,q \label{eq100}
\end{align}
\ENDFOR
\end{algorithmic}
\end{algorithm}

There are $\bar{b}_{mp}^{(k)}$, $\bar{c}_{qn}^{(k)}$, and $\bar{s}_{pq}^{(k)}$ in algorithm \ref{algorithm8} which are modifications to avoid zero locking problems. The following gives their definitions.
\begin{align*}
\bar{b}_{mp}^{(k)} &\equiv \left\{
  \begin{array}{rl}
    b_{mp}^{(k)}\hspace{13 mm} & \text{if  } \nabla_{\mathbf{B}}J\big(\mathbf{B}^{(k)},\mathbf{S}^{(k)},\mathbf{C}^{(k)}\big)_{mp} \ge 0 \\
    \max(b_{mp}^{(k)}, \sigma) & \text{otherwise}\end{array}, \right. \\
\bar{c}_{qn}^{(k)} &\equiv \left\{
  \begin{array}{rl}
    c_{qn}^{(k)}\hspace{13 mm} & \text{if  } \nabla_{\mathbf{C}}J\big(\mathbf{B}^{(k+1)},\mathbf{S}^{(k)},\mathbf{C}^{(k)}\big)_{qn} \ge 0 \\
    \max(c_{qn}^{(k)}, \sigma) & \text{otherwise}\end{array}, \right. \\
\bar{s}_{pq}^{(k)} &\equiv \left\{
  \begin{array}{rl}
    s_{pq}^{(k)}\hspace{13 mm} & \text{if  } \nabla_{\mathbf{S}}J\big(\mathbf{B}^{(k+1)},\mathbf{S}^{(k)},\mathbf{C}^{(k+1)}\big)_{pq} \ge 0 \\
    \max(s_{pq}^{(k)}, \sigma) & \text{otherwise}\end{array}, \right. 
\end{align*}
with $\sigma$ is a small positive number, $\mathbf{\bar{B}}$, $\mathbf{\bar{C}}$, and $\mathbf{\bar{S}}$ are matrices that contain $\bar{b}_{mp}$, $\bar{c}_{qn}$, and $\bar{s}_{pq}$ respectively. And there are also the variables $\delta_{\mathbf{B}}$, $\delta_{\mathbf{C}}$, and $\delta_{\mathbf{S}}$ in algorithm \ref{algorithm8} that play a crucial role in guaranteeing the convergence of the algorithm (see appendix for the details).

Algorithm \ref{algorithm9} shows modifications to algorithm \ref{algorithm8} in order to guarantee the convergence as suggested by theorem \ref{theorem30}, \ref{theorem31}, and \ref{theorem32} in appendix with step is a constant that determine how fast $\delta_{\mathbf{B}}^{(k)}$, $\delta_{\mathbf{C}}^{(k)}$, and $\delta_{\mathbf{S}}^{(k)}$ grow in order to satisfy the nonincreasing property. Note that we set the same step value for all sequences, but different values can also be employed. 

\begin{algorithm}
\caption{A Convergent algorithm for BNMtF problem.}
\label{algorithm9}
\begin{algorithmic}
\STATE Initialize $\mathbf{B}^{(0)}$, $\mathbf{C}^{(0)}$, and $\mathbf{S}^{(0)}$ with nonnegative matrices, and choose a small positive number for $\delta$ and an integer number for step.
\FOR {$k=0,\ldots,K$}
\STATE $ $
\STATE $\delta_{\mathbf{B}}^{(k)} \longleftarrow \delta$
\REPEAT
\STATE \begin{align*} b_{mp}^{(k+1)} \longleftarrow & \;b_{mp}^{(k)} - \frac{\bar{b}_{mp}^{(k)}\times\nabla_{\mathbf{B}}J(\mathbf{B}^{(k)},\mathbf{S}^{(k)},\mathbf{C}^{(k)})_{mp}}{\big(\mathbf{\bar{B}}^{(k)}\mathbf{S}^{(k)}\mathbf{C}^{(k)}\mathbf{C}^{(k)T}\mathbf{S}^{(k)T}+\beta\mathbf{\bar{B}}^{(k)}\mathbf{\bar{B}}^{(k)T}\mathbf{\bar{B}}^{(k)}\big)_{mp}+\delta_{\mathbf{B}}^{(k)}}\;\;\forall m,p 
\\ 
\delta_{\mathbf{B}}^{(k)} \longleftarrow & \;\delta_{\mathbf{B}}^{(k)}\times \mathrm{step} \end{align*}
\UNTIL {$J\big( \mathbf{B}^{(k+1)},\mathbf{S}^{(k)}, \mathbf{C}^{(k)} \big) \le J\big(\mathbf{B}^{(k)},\mathbf{S}^{(k)}, \mathbf{C}^{(k)} \big)$}
\STATE $ $
\STATE $\delta_{\mathbf{C}}^{(k)} \longleftarrow \delta$
\REPEAT
\STATE \begin{align*} c_{qn}^{(k+1)} \longleftarrow & \;c_{qn}^{(k)} - \frac{\bar{c}_{qn}^{(k)}\times\nabla_{\mathbf{C}}J(\mathbf{B}^{(k+1)},\mathbf{S}^{(k)},\mathbf{C}^{(k)})_{qn}}{\big(\mathbf{S}^{(k)T}\mathbf{B}^{(k+1)T}\mathbf{B}^{(k+1)}\mathbf{S}^{(k)}\mathbf{\bar{C}}^{(k)}+\alpha\mathbf{\bar{C}}^{(k)}\mathbf{\bar{C}}^{(k)T}\mathbf{\bar{C}}^{(k)}\big)_{qn}+\delta_{\mathbf{C}}^{(k)}} \;\;\forall q,n 
\\ 
\delta_{\mathbf{C}}^{(k)} \longleftarrow & \;\delta_{\mathbf{C}}^{(k)}\times \mathrm{step} \end{align*}
\UNTIL {$J\big( \mathbf{B}^{(k+1)},\mathbf{S}^{(k)}, \mathbf{C}^{(k+1)} \big) \le J\big(\mathbf{B}^{(k+1)},\mathbf{S}^{(k)}, \mathbf{C}^{(k)} \big)$}
\STATE $ $
\STATE $\delta_{\mathbf{S}}^{(k)} \longleftarrow \delta$
\REPEAT
\STATE \begin{align*} s_{pq}^{(k+1)} \longleftarrow & \;s_{pq}^{(k)} - \frac{\bar{s}_{pq}^{(k)}\times\nabla_{\mathbf{S}}J(\mathbf{B}^{(k+1)},\mathbf{S}^{(k)},\mathbf{C}^{(k+1)})_{pq}}{\big(\mathbf{B}^{(k+1)T}\mathbf{B}^{(k+1)}\mathbf{\bar{S}}^{(k)}\mathbf{C}^{(k+1)}\mathbf{C}^{(k+1)T}\big)_{pq}+\delta_{\mathbf{S}}^{(k)}} \;\;\forall p,q \\ \delta_{\mathbf{S}}^{(k)} \longleftarrow & \;\delta_{\mathbf{S}}^{(k)}\times \mathrm{step} \end{align*}
\UNTIL {$J\big( \mathbf{B}^{(k+1)},\mathbf{S}^{(k+1)}, \mathbf{C}^{(k+1)} \big) \le J\big(\mathbf{B}^{(k+1)},\mathbf{S}^{(k)}, \mathbf{C}^{(k+1)} \big)$}
\ENDFOR
\end{algorithmic}
\end{algorithm}

\section{Experimental results} \label{results}

We will now analyze the convergence of the proposed algorithms \ref{algorithm7} (\textbf{MU-B}) and \ref{algorithm9} (\textbf{AU-B}) numerically. However, as it is generally difficult to reach stationary point in an acceptable computational time, only the nonincreasing property (or lack of it) will be shown. And because the BNMtF was originally designed for clustering purpose, we will also analyze this property. As dataset, we use Reuters-21578 data corpus. The dataset is especially interesting because many NMF-based clustering methods are tested using it, e.g.: \cite{Shahnaz,Ding1,Xu}. Detailed discussion about the dataset and preprocessing steps can be found in our previous work \cite{Mirzal2}. In summary, datasets were formed by combining top 2, 4, 6, 8, 10, and 12 classes from the corpus. Table \ref{ch2:table3} summarizes the statistics of these test datasets, where \#doc, \#word, \%nnz, max, and min refer to the number of documents, the number of words, percentage of nonzero entries, maximum cluster size, and minimum cluster size respectively. And as the corpus is bipartite, clustering can be done either for documents or words. We will evaluate both document clustering and word clustering. For comparison, we use the following algorithms:
\begin{itemize}
\item standard NMF algorithm \cite{Lee2} $\rightarrow$ \textbf{LS},
\item original UNMF algorithm \cite{Ding1} $\rightarrow$ \textbf{D-U},
\item original BNMtF algorithm \cite{Ding1} $\rightarrow$ \textbf{D-B},
\item MUR based algorithm for UNMF, i.e.: algorithm 3 in \cite{Mirzal2} $\rightarrow$ \textbf{MU-U}, and
\item convergent algorithm for UNMF proposed in our previous work, i.e.: algorithm 4 in \cite{Mirzal2} $\rightarrow$ \textbf{AU-U}.
\end{itemize}
All codes were written in Octave and executed under Linux platform using a notebook with 1.86 GHz Intel processor and 2 GB RAM.

\begin{table}
  \begin{center}
    \caption{Statistics of the test datasets.}
    \centering
    \begin{tabular}{lrrrrr}
    \hline
    Dataset &\#doc & \#word & \%nnz & max & min \\
    \hline
    Reuters2    & 6090 & 8547  & 0.363 & 3874 & 2216 \\
    Reuters4    & 6797 & 9900  & 0.353 & 3874 & 333 \\
    Reuters6    & 7354 & 10319 & 0.347 & 3874 & 269 \\
    Reuters8    & 7644 & 10596 & 0.340 & 3874 & 144 \\
    Reuters10   & 7887 & 10930 & 0.336 & 3874 & 114 \\
    Reuters12   & 8052 & 11172 & 0.333 & 3874 &  75 \\
    \hline
    \end{tabular}
    \label{ch2:table3}
  \end{center}
\end{table}

\subsection{The nonincreasing property} \label{nonincreasing}

Figures \ref{fig5}--\ref{fig8} show graphically the nonincreasing property (or lack of it) of MU-B and AU-B. Because there are two adjustable parameters, $\alpha$ and $\beta$, we fix one parameter while studying the other. Figure \ref{fig5} and \ref{fig6} show the results for fixed $\beta=1$, and figure \ref{fig7} and \ref{fig8} for fixed $\alpha=1$. As shown, while MU-B fails to show the nonincreasing property for large $\alpha$ and $\beta$ values, AU-B successfully preserves this property regardless of $\alpha$ and $\beta$ values. Note that we set $\delta=\sigma=10^{-8}$, and $\mathrm{step}=10$ for MU-B and AU-B in all experiments. A similar study for UNMF algorithms can be found in our previous work \cite{Mirzal2}.

\begin{figure}
 \begin{center}
  \subfigure[Small $\alpha$]{
   \includegraphics[width=0.45\textwidth]{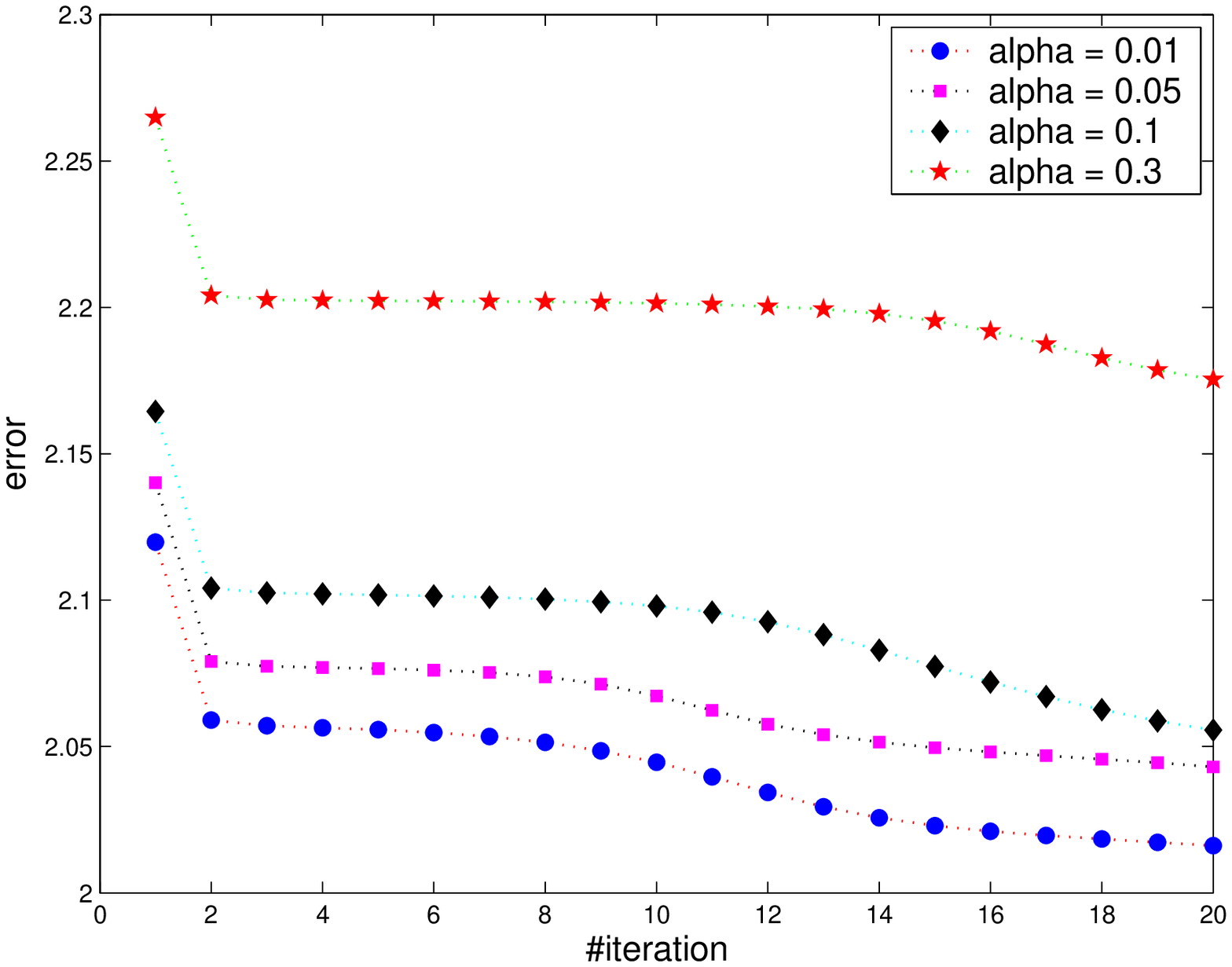}
   \label{fig5a}
  }
  \subfigure[Medium to large $\alpha$]{
   \includegraphics[width=0.45\textwidth]{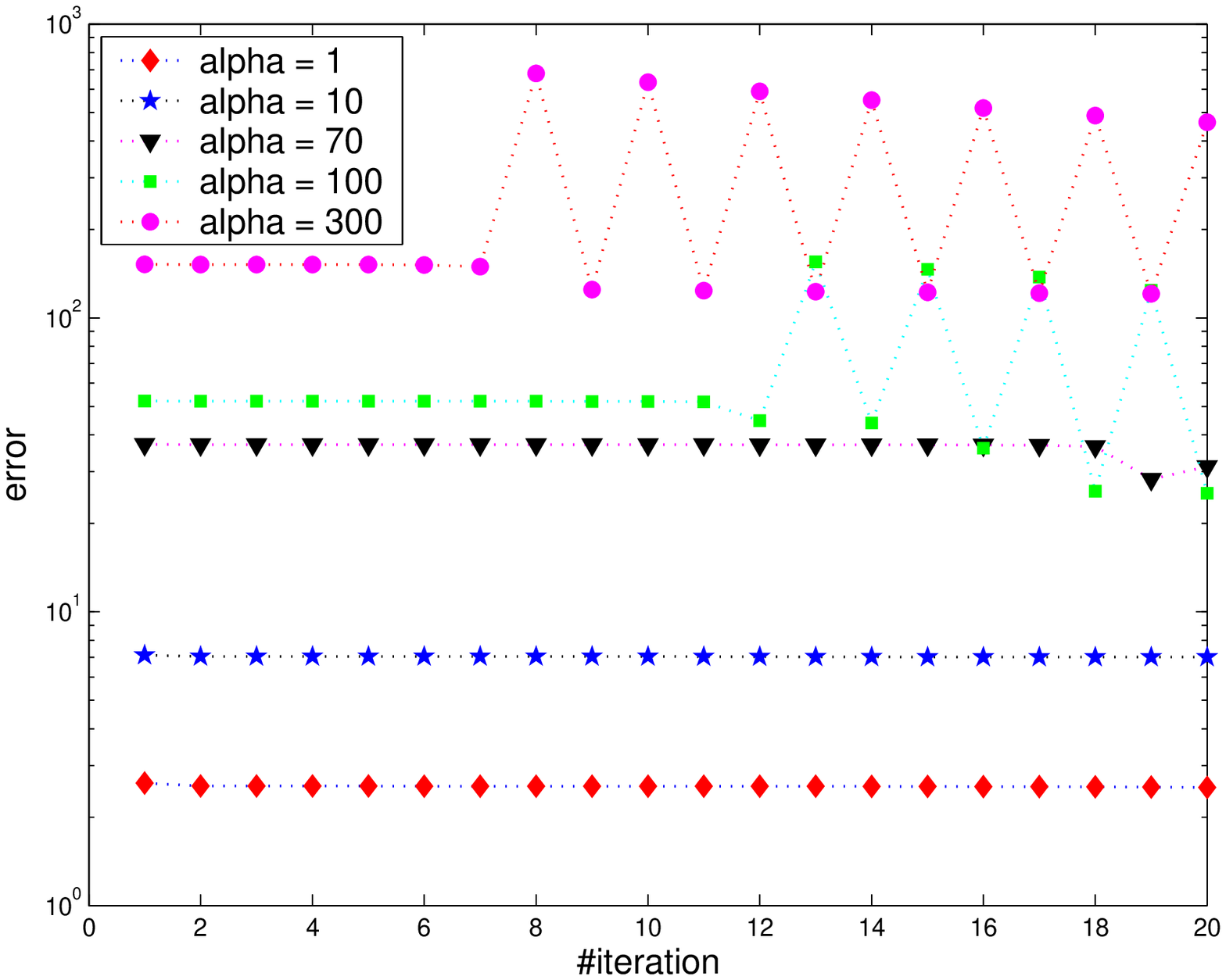}
   \label{fig5b}
  }
\\
  \subfigure[Some values of $\alpha$]{
   \includegraphics[width=0.7\textwidth]{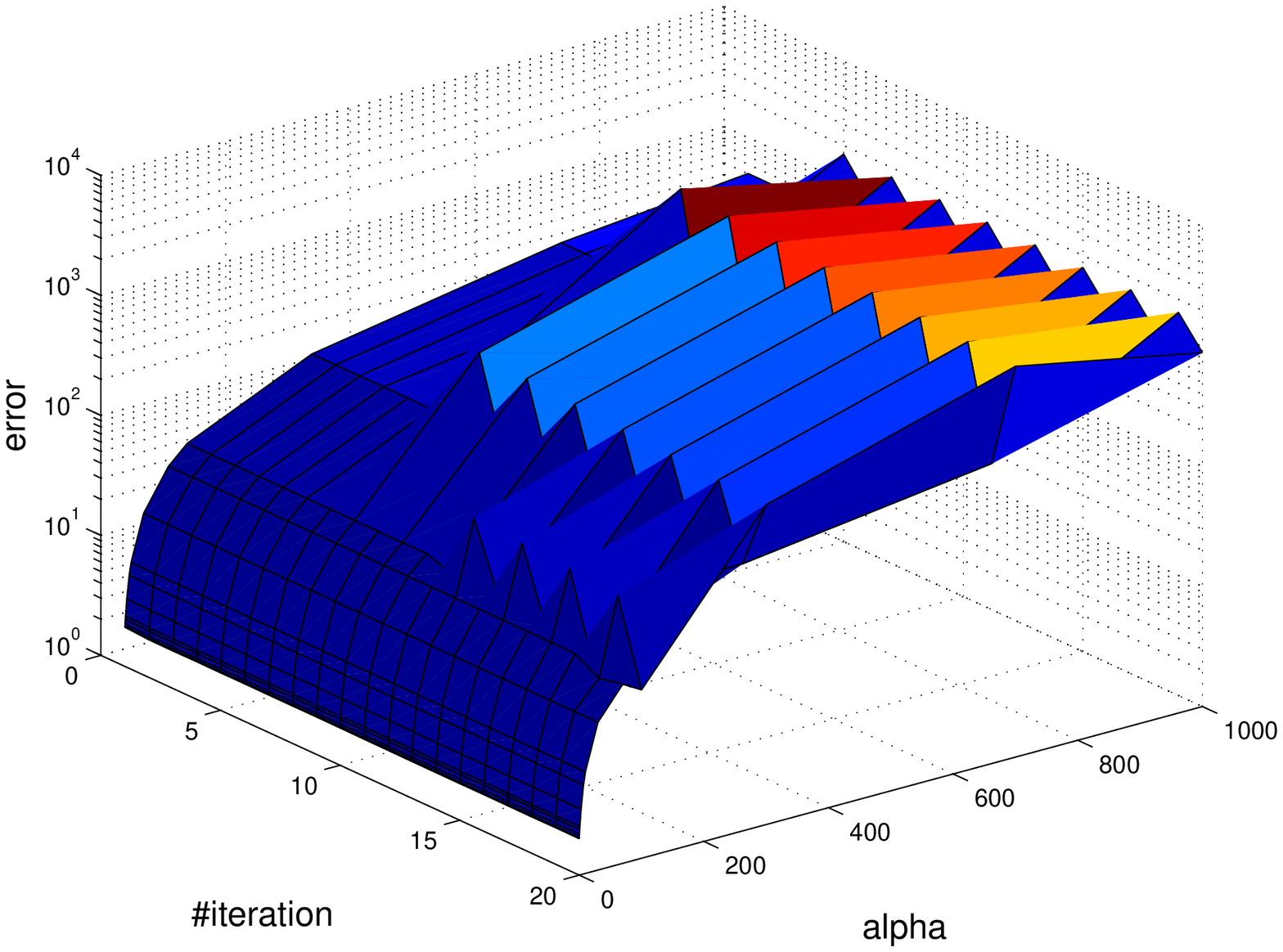}
   \label{fig5c}
  }
  \caption{MU-B($\alpha$) error per iteration for Reuters4 dataset ($\beta=1$).}
  \label{fig5}
 \end{center}
\end{figure}

\begin{figure}
 \begin{center}
  \subfigure[Small $\alpha$]{
   \includegraphics[width=0.45\textwidth]{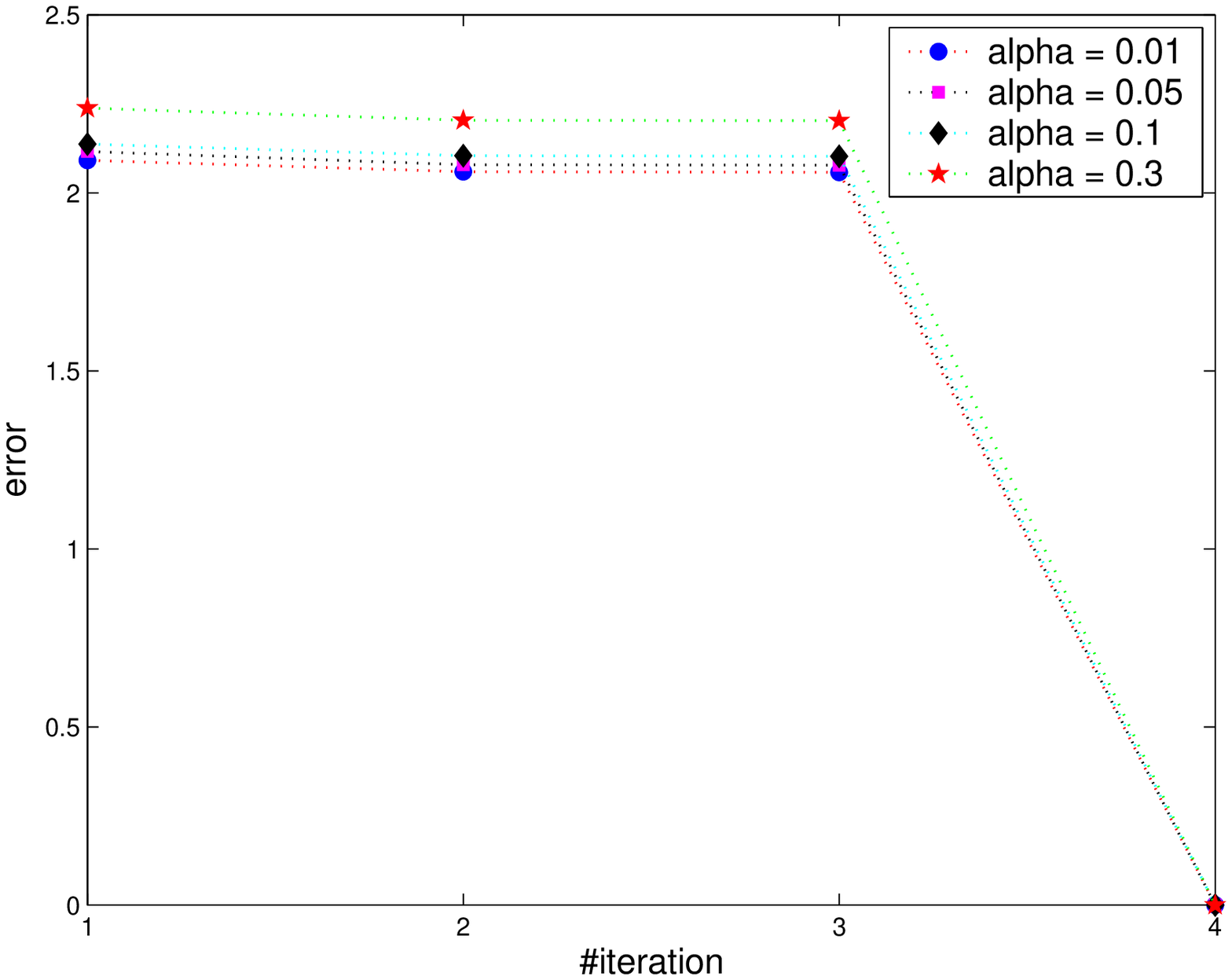}
   \label{fig6a}
  }
  \subfigure[Medium to large $\alpha$]{
   \includegraphics[width=0.45\textwidth]{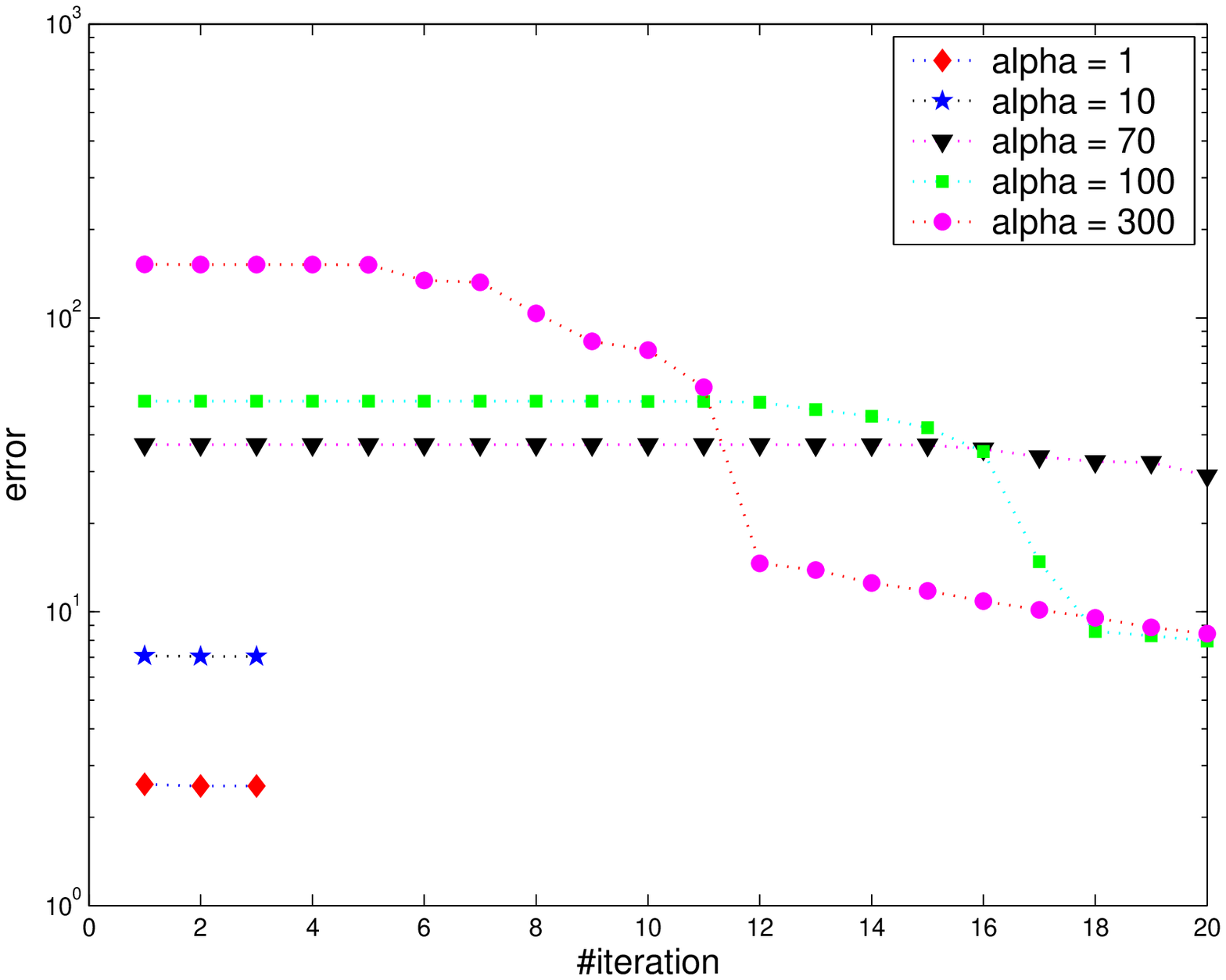}
   \label{fig6b}
  }
\\
  \subfigure[Some values of $\alpha$]{
   \includegraphics[width=0.7\textwidth]{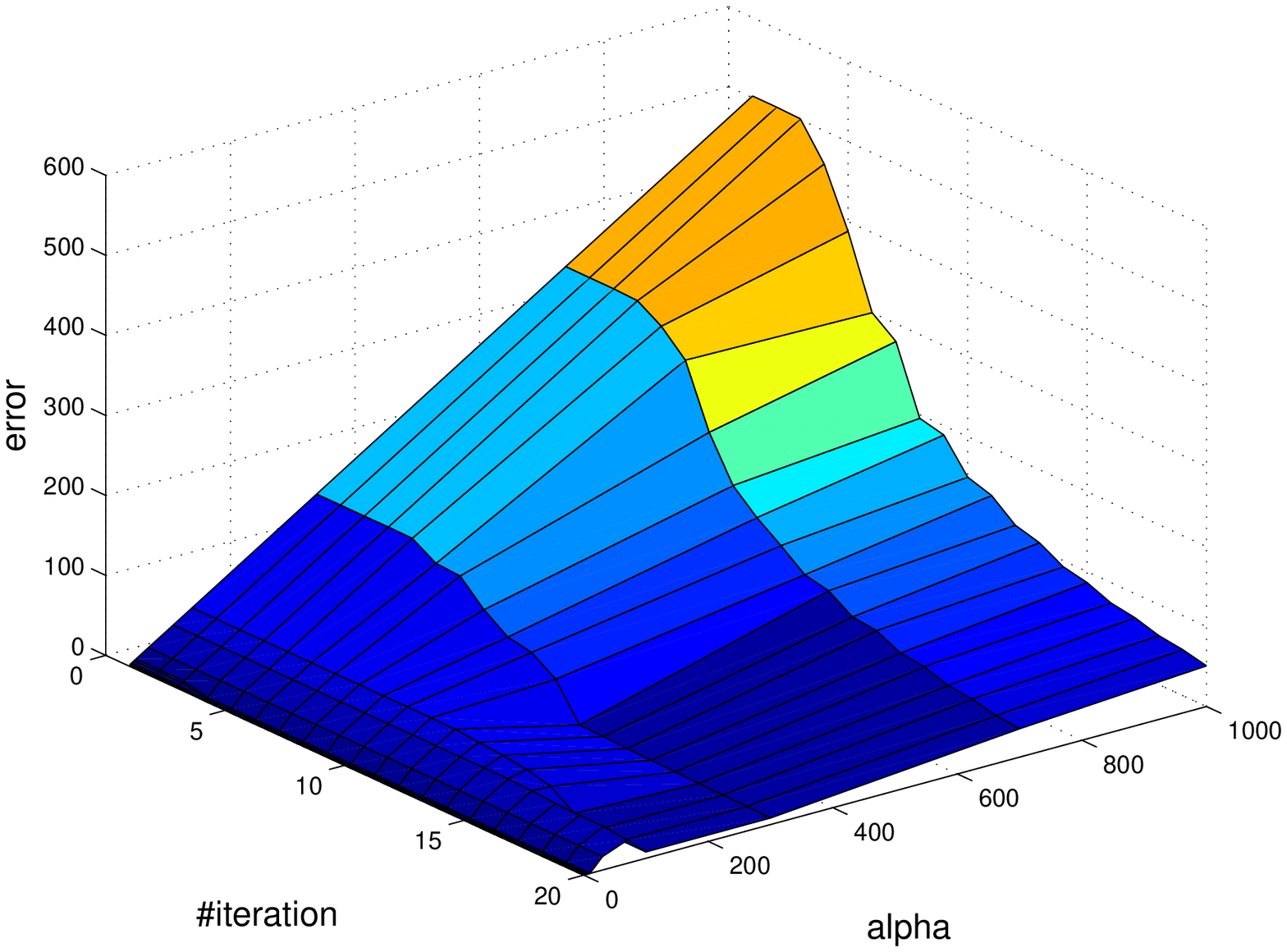}
   \label{fig6c}
  }
  \caption{AU-B($\alpha$) error per iteration for Reuters4 dataset ($\beta=1$).}
  \label{fig6}
 \end{center}
\end{figure}

\begin{figure}
 \begin{center}
  \subfigure[Small $\beta$]{
   \includegraphics[width=0.45\textwidth]{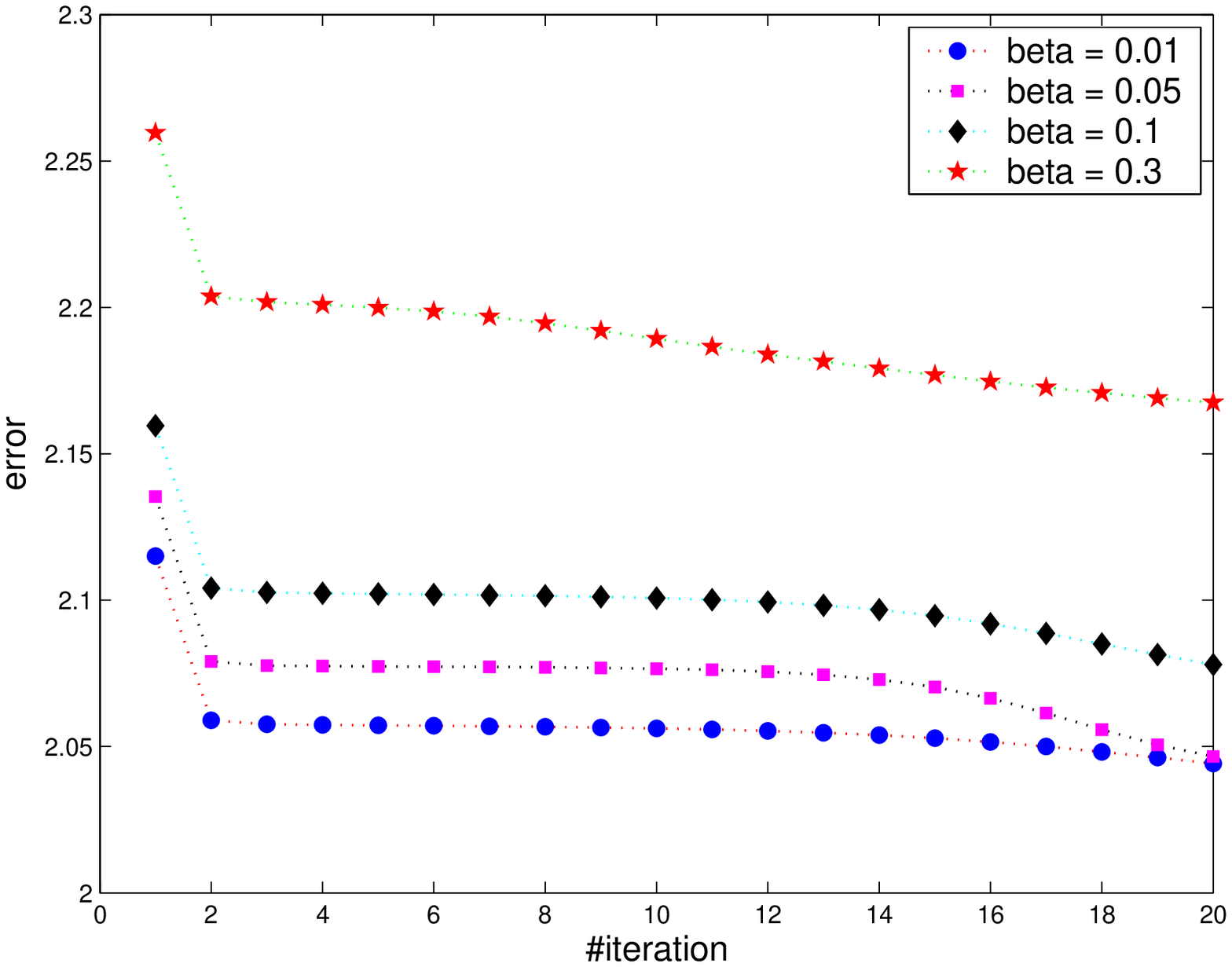}
   \label{fig7a}
  }
  \subfigure[Medium to large $\beta$]{
   \includegraphics[width=0.45\textwidth]{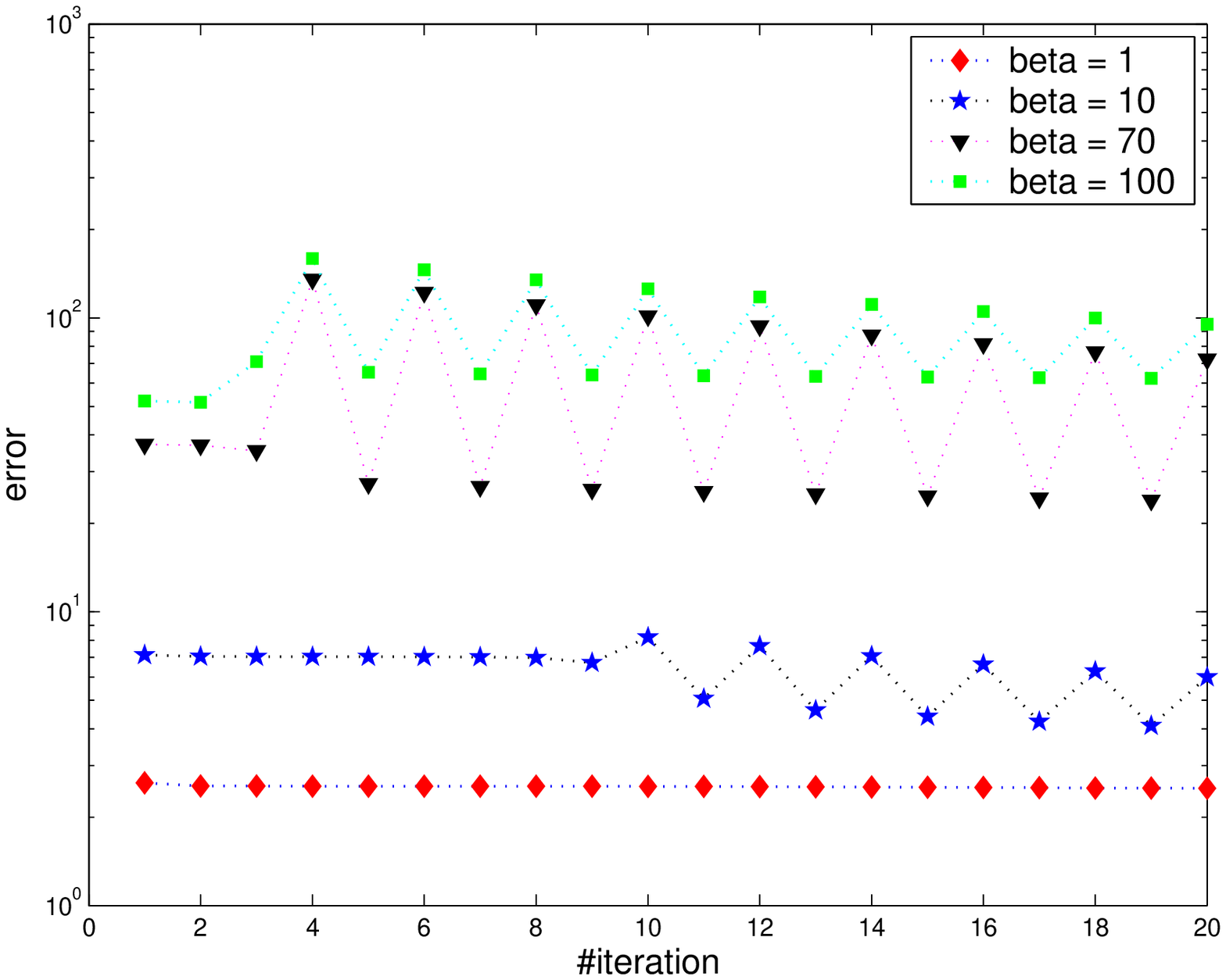}
   \label{fig7b}
  }
\\
  \subfigure[Some values of $\beta$]{
   \includegraphics[width=0.7\textwidth]{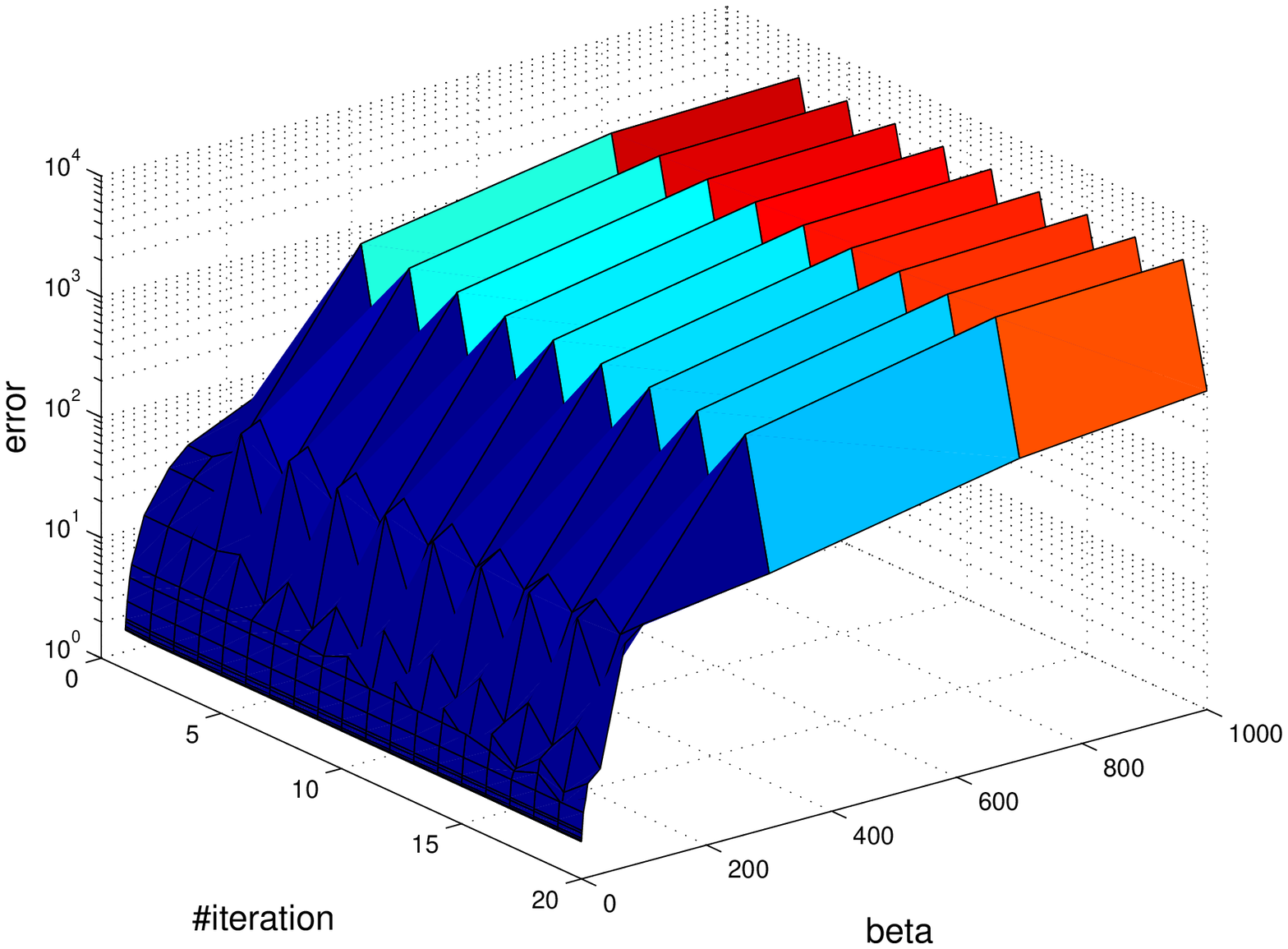}
   \label{fig7c}
  }
  \caption{MU-B($\beta$) error per iteration for Reuters4 dataset ($\alpha=1$).}
  \label{fig7}
 \end{center}
\end{figure}

\begin{figure}
 \begin{center}
  \subfigure[Small $\beta$]{
   \includegraphics[width=0.45\textwidth]{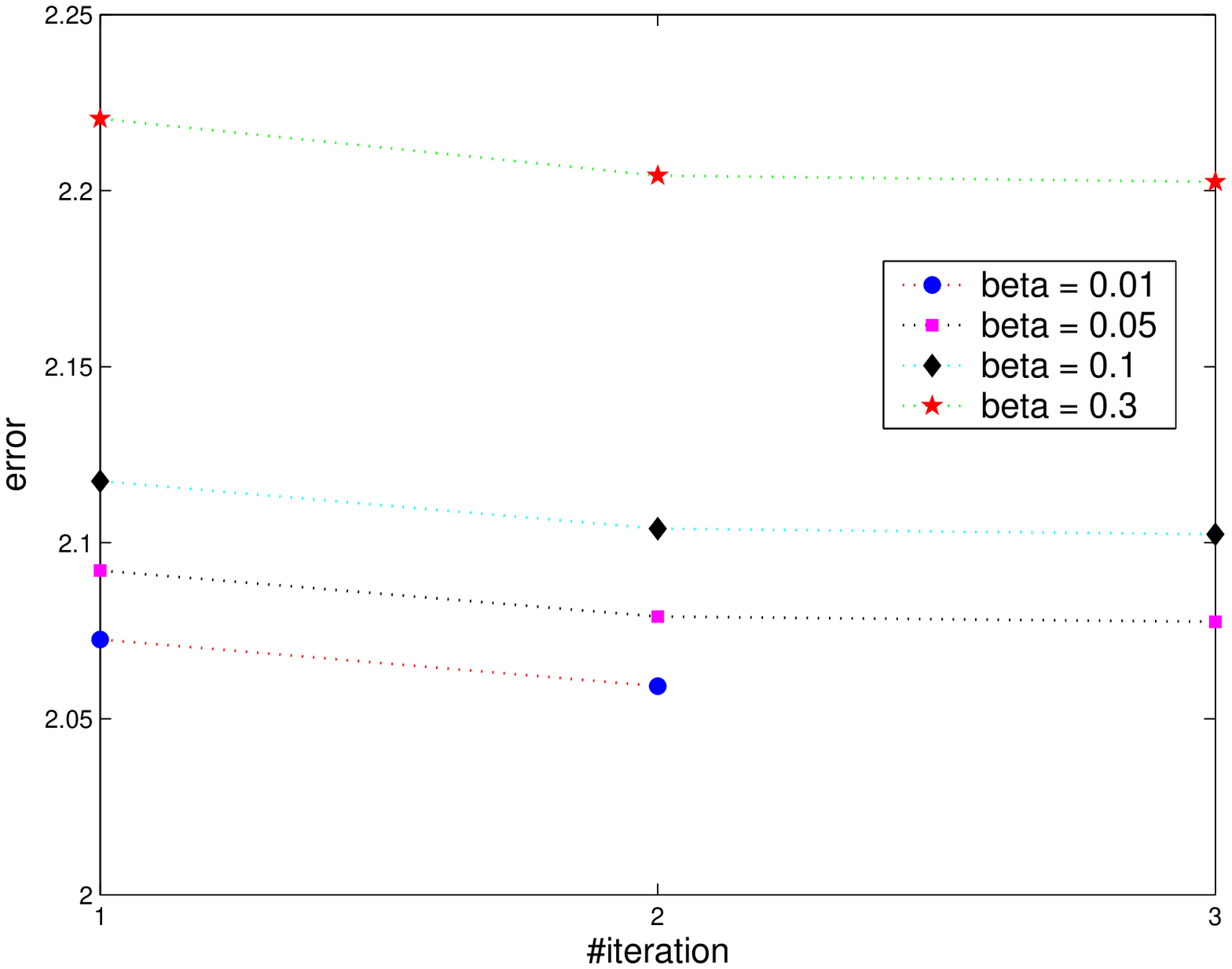}
   \label{fig8a}
  }
  \subfigure[Medium to large $\beta$]{
   \includegraphics[width=0.45\textwidth]{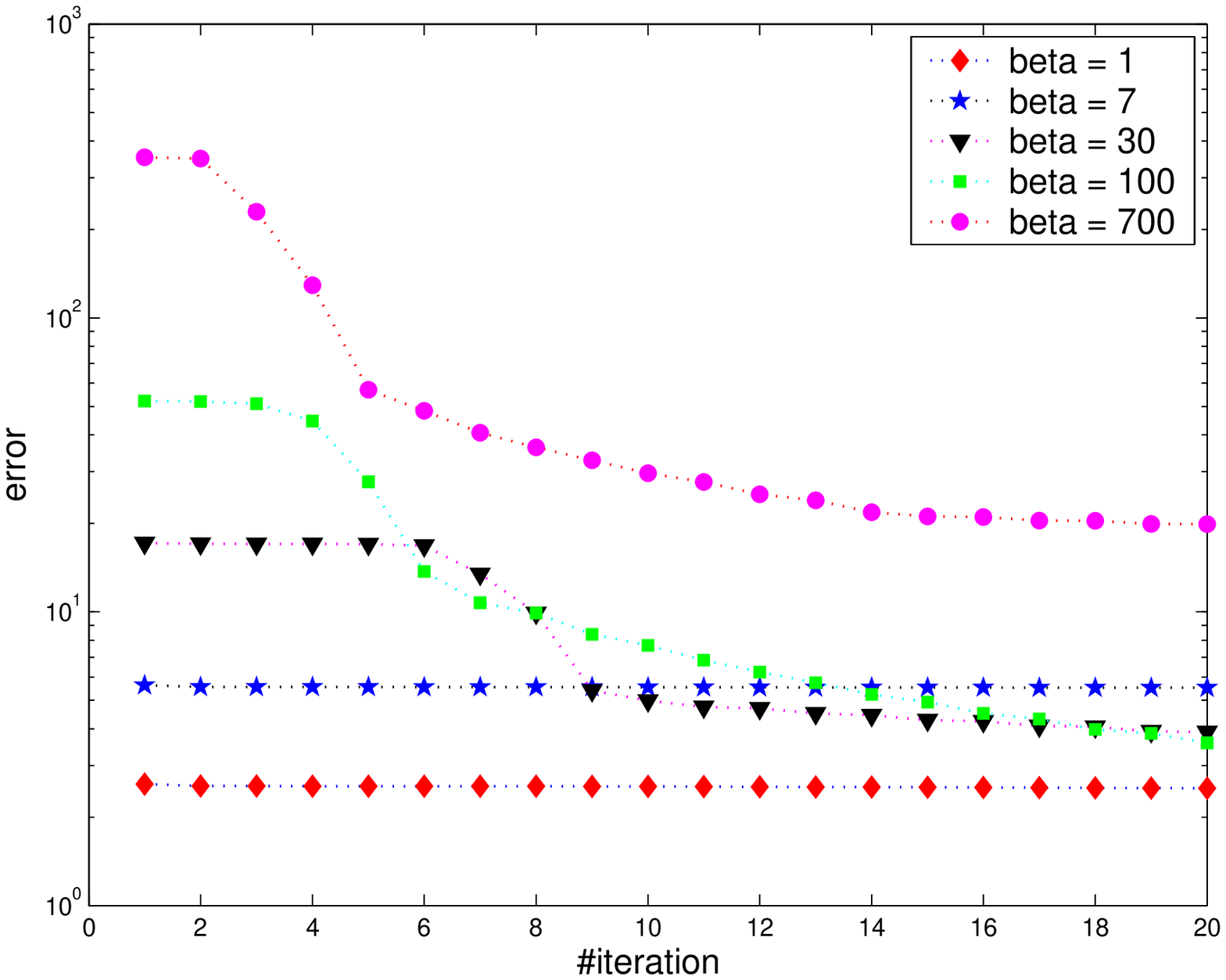}
   \label{fig8b}
  }
\\
  \subfigure[Some values of $\beta$]{
   \includegraphics[width=0.7\textwidth]{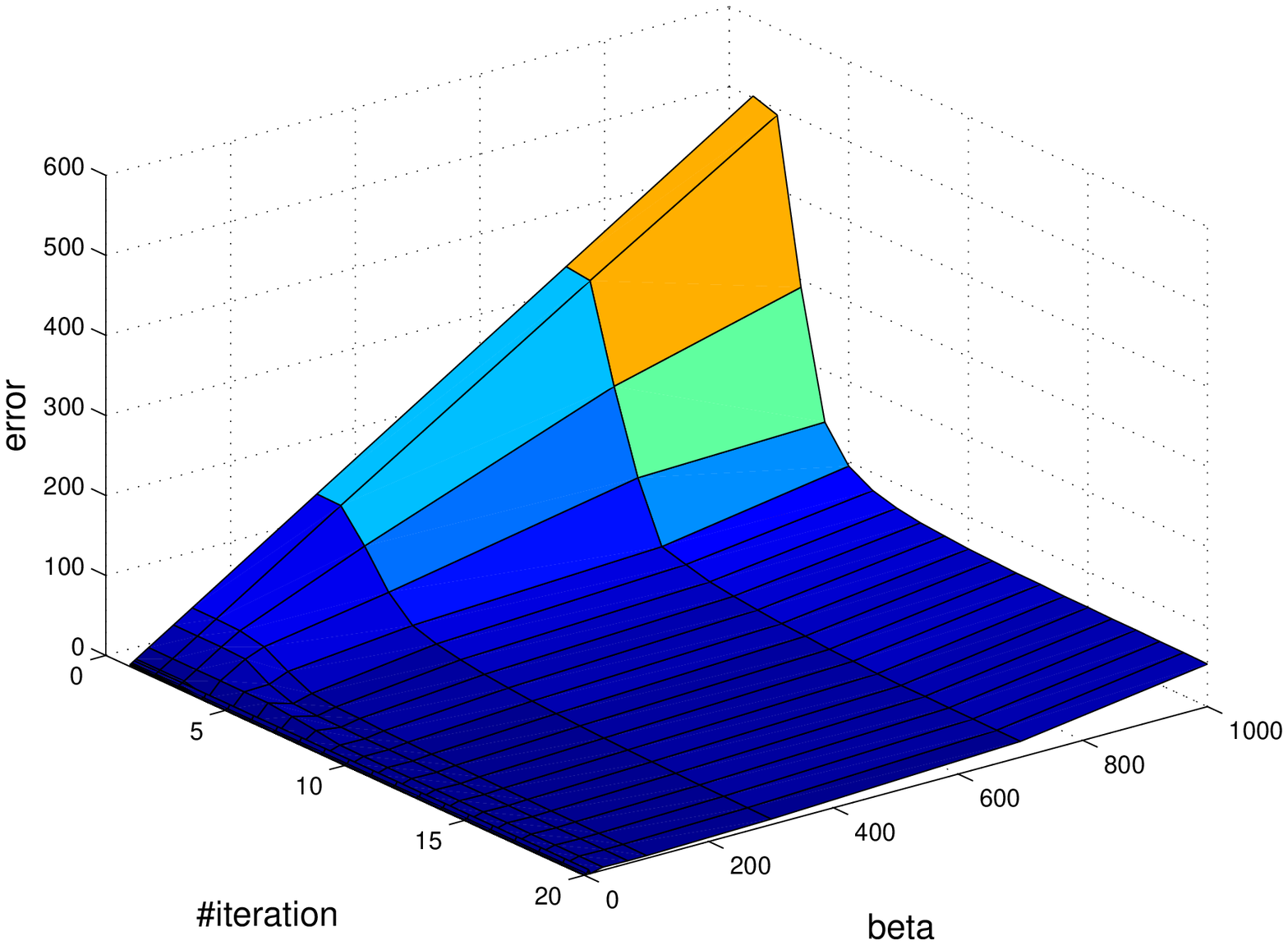}
   \label{fig8c}
  }
  \caption{AU-B($\beta$) error per iteration for Reuters4 dataset ($\alpha=1$).}
  \label{fig8}
 \end{center}
\end{figure}

Not surprisingly, there are computational consequence for this property for large $\alpha$ and/or $\beta$. Table \ref{table3} shows time comparisons between these algorithms for Reuters4 dataset. Note that, $\alpha$ or $\beta$ letter is appended to the algorithm's acronyms to indicate which parameter is being varied. For example AU-B($\alpha$) means AU-B with fixed $\beta$ and varied $\alpha$. Note that LS, D-U, and D-B do not have these parameters. As shown, the computational times of MU algorithms practically are independent from $\alpha$ and $\beta$ values. And AU algorithms seem to become slower for some large $\alpha$ or $\beta$. This probably because for large $\alpha$ or $\beta$ values the AU algorithms execute the inner iterations (shown as $\mathbf{repeat}$ $\mathbf{until}$ loops). Also, there are some anomalies in the AU-B($\alpha$) and AU-B($\beta$) cases where for some $\alpha$ or $\beta$ values, execution times are unexpectedly very fast. To investigate these, we display the number of iteration (\#iter) and the number of inner iteration (\#initer) for AU algorithms in table \ref{table4} (note that MU algorithms reach maximum predefined number of iteration for all cases: 20 iterations). As displayed, in the cases where AU algorithms performed worse than their MU counterparts, they executed the inner iterations. And when an AU algorithm performed better, then its \#iter is smaller than \#iter of the corresponding MU algorithm and the inner iteration was not executed. These explain the differences in computational times in table \ref{table3}.

\begin{table}
\renewcommand{\arraystretch}{1}
  \begin{center}
    \caption{Time comparison (in second) for Reuters4 dataset.}
    \centering
    \begin{tabular}{|r|r|r|r|r|r|r|}
    \hline
    $\alpha$ or $\beta$ & MU-U & AU-U & MU-B($\alpha$) & AU-B($\alpha$) & MU-B($\beta$) & AU-B($\beta$)\\
    \hline

    0.01   & 110 & 110 & 121 & 41.1 & 122 & 27.2\\
    0.05   & 110 & 110 & 121 & 40.9 & 121 & 40.7\\
     0.1   & 109 & 109 & 121 & 40.8 & 121 & 41.2\\
     0.3   & 110 & 109 & 121 & 40.4 & 121 & 41.1\\
     0.7   & 110 & 110 & 121 & 272  & 121 & 41.2\\
       1   & 110 & 110 & 121 & 40.8 & 121 & 273\\
       3   & 110 & 110 & 121 & 40.4 & 121 & 40.7\\
       7   & 110 & 110 & 121 & 40.4 & 121 & 273\\
      10   & 110 & 110 & 121 & 40.8 & 121 & 41.1\\
      30   & 109 & 110 & 121 & 272  & 121 & 442\\
      70   & 109 & 137 & 121 & 332  & 121 & 525\\
     100   & 110 & 232 & 121 & 382  & 121 & 605\\
     300   & 110 & 232 & 121 & 514  & 121 & 579\\
     700   & 110 & 461 & 121 & 607  & 121 & 606\\
    1000   & 110 & 411 & 121 & 606  & 121 & 365\\
    \hline
    \end{tabular}
    \label{table3}
  \end{center}
\end{table}

\begin{table}
\renewcommand{\arraystretch}{1}
  \begin{center}
    \caption{\#iter and \#initer of AU algorithms (Reuters4).}
    \centering
    \begin{tabular}{|r|r|r|r|r|r|r|}
    \hline
    $\alpha$ or $\beta$ & AU-U & AU-B($\alpha$) & AU-B($\beta$)\\
    \hline
    & \#iter / \#initer & \#iter / \#initer & \#iter / \#initer \\
    \hline
    0.01   & 20 / 0 &  3 / 0 &  2 / 0 \\
    0.05   & 20 / 0 &  3 / 0 &  3 / 0 \\
     0.1   & 20 / 0 &  3 / 0 &  3 / 0 \\
     0.3   & 20 / 0 &  3 / 0 &  3 / 0 \\
     0.7   & 20 / 0 & 20 / 0 &  3 / 0 \\
       1   & 20 / 0 &  3 / 0 & 20 / 0 \\
       3   & 20 / 0 &  3 / 0 &  3 / 0 \\
       7   & 20 / 0 &  3 / 0 & 20 / 0 \\
      10   & 20 / 0 &  3 / 0 &  3 / 0 \\
      30   & 20 / 0 & 20 / 0 & 20 / 44 \\
      70   & 20 / 7 & 20 / 23 & 20 / 66 \\
     100   & 20 / 32 & 20 / 22 & 20 / 88 \\
     300   & 20 / 32 & 20 / 65 & 20 / 81 \\
     700   & 20 / 92 & 20 / 75 & 20 / 88 \\
    1000   & 20 / 79 & 20 / 90 & 20 / 24 \\
    \hline
    \end{tabular}
    \label{table4}
  \end{center}
\end{table}

\subsection{The minimization slopes}

The minimization slopes of MU and AU algorithms are important to be studied as the algorithms can be slow to settle. As shown by Lin \cite{CJLin}, LS is very fast to minimize the objective for some first iterations and then settles. In table \ref{table5}, we display errors for some first iterations for LS, MU-U, AU-U, MU-B, and AU-B (D-U and D-B do not have the nonincreasing property). Note that error0 refers to the initial error before the algorithms start running, and error$n$ is the error at $n$-th iteration. As shown, all algorithms are exceptionally very good at reducing errors in the first iterations. But then, the improvements are rather negligible with respect to the corresponding first improvements and the sizes of the datasets. Accordingly, we set maximum number of iteration to 20 for the next experiments. Note that, in this case AU-B converged at the third iteration.

\begin{table}
 \begin{center}
   \caption{Errors for some first iterations (Reuters4).}
   \centering
   \begin{tabular}{|l|r|r|r|r|r|r|}
   \hline
   & error0 & error1 & error2 & error3 & error4 & error5 \\
   \hline

   LS   & 1373  & 0.476 & 0.474 & 0.472 & 0.469 & 0.466 \\
   MU-U & 4652  & 1.681 & 1.603 & 1.596 & 1.591 & 1.583 \\
   AU-U & 4657  & 1.681 & 1.605 & 1.595 & 1.586 & 1.573 \\
   MU-B & 12474 & 2.164 & 2.104 & 2.103 & 2.102 & 2.102 \\
   AU-B & 12680 & 2.137 & 2.104 & 2.103 & -     & -     \\
   \hline
  \end{tabular}
  \label{table5}
 \end{center}
\end{table}

\subsection{Determining $\alpha$ and $\beta$}

In the proposed algorithms, there are two dataset-dependent parameters, $\alpha$ and $\beta$, that have to be learned first. Because orthogonal NMFs are introduced to improve clustering capability of the standard NMF \cite{Ding1}, these parameters will be learned based on clustering results on test dataset. We use Reuters4 for this purpose. These parameters do not exist in the original orthogonal NMFs nor in the other orthogonal NMF algorithms \cite{Yoo1,Yoo2,Choi}. However, we notice that our formulations resemble sparse NMF formulation \cite{HKim2,JKim,HKim}, or in general case also known as constrained NMF \cite{Pauca}. As shown in \cite{HKim2,HKim,JKim}, sparse NMF usually can give good results if $\alpha$ and/or $\beta$ are rather small positive numbers. To determine $\alpha$ and $\beta$, we evaluate clustering qualities produced by our algorithms as $\alpha$ or $\beta$ values grow measured by the standard clustering metrics: \emph{mutual information} (MI), \emph{entropy} (E), \emph{purity} (P), and \emph{Fmeasure} (F). The detailed discussions on these metrics can be found in \cite{Mirzal2}. Note that while larger MI, F, and P indicate better results, smaller E indicates better results. As shown in figure \ref{fig9}, for UNMF algorithms (MU-U and AU-U) $\alpha=0.1$ seems to be a good choice. For MU-B it seems that $\alpha=0.1$ and $\beta=3$ are acceptable settings. And for AU-B, $\alpha=0.7$ and $\beta=1$ seem to be good settings. Based on this results, we decided to set $\alpha=0.1$ and $\beta=1$ for all datasets and algorithms.

\begin{figure}
 \begin{center}

  \subfigure[MU-U]{
   \includegraphics[width=0.45\textwidth]{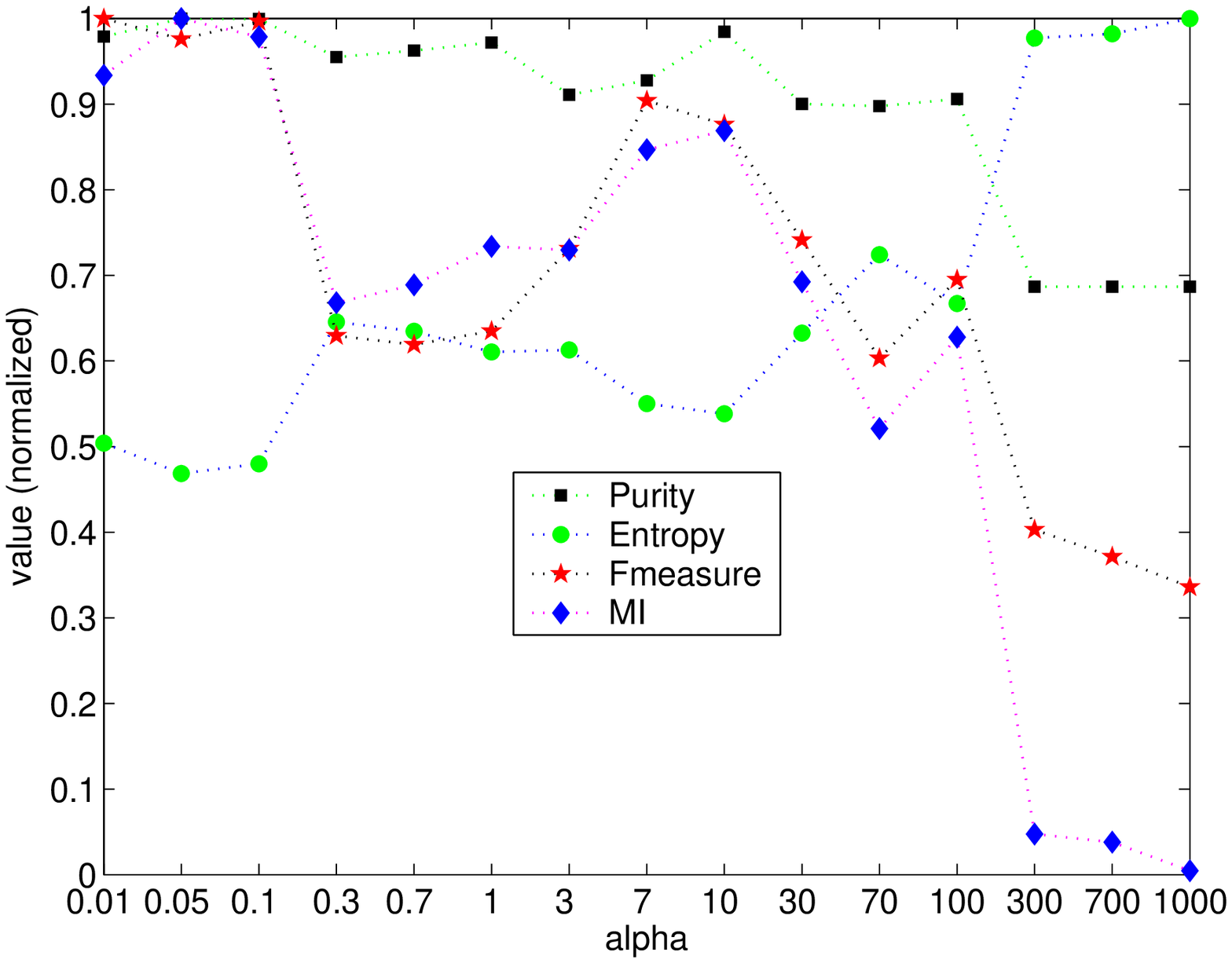}
   \label{fig9a}
  }
  \subfigure[AU-U]{
   \includegraphics[width=0.45\textwidth]{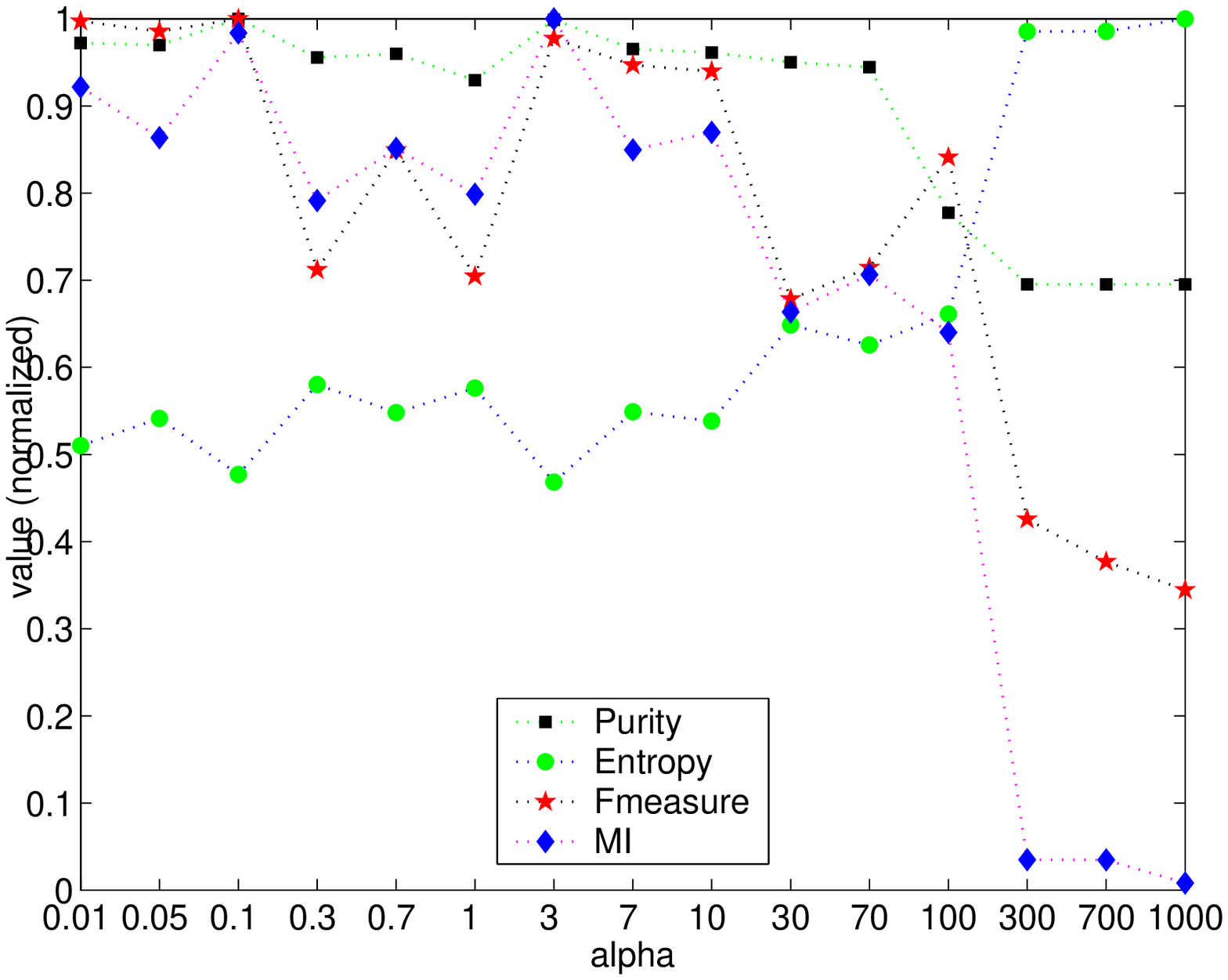}
   \label{fig9b}
  } 
\\
  \subfigure[MU-B($\alpha$), $\beta=1$]{
   \includegraphics[width=0.45\textwidth]{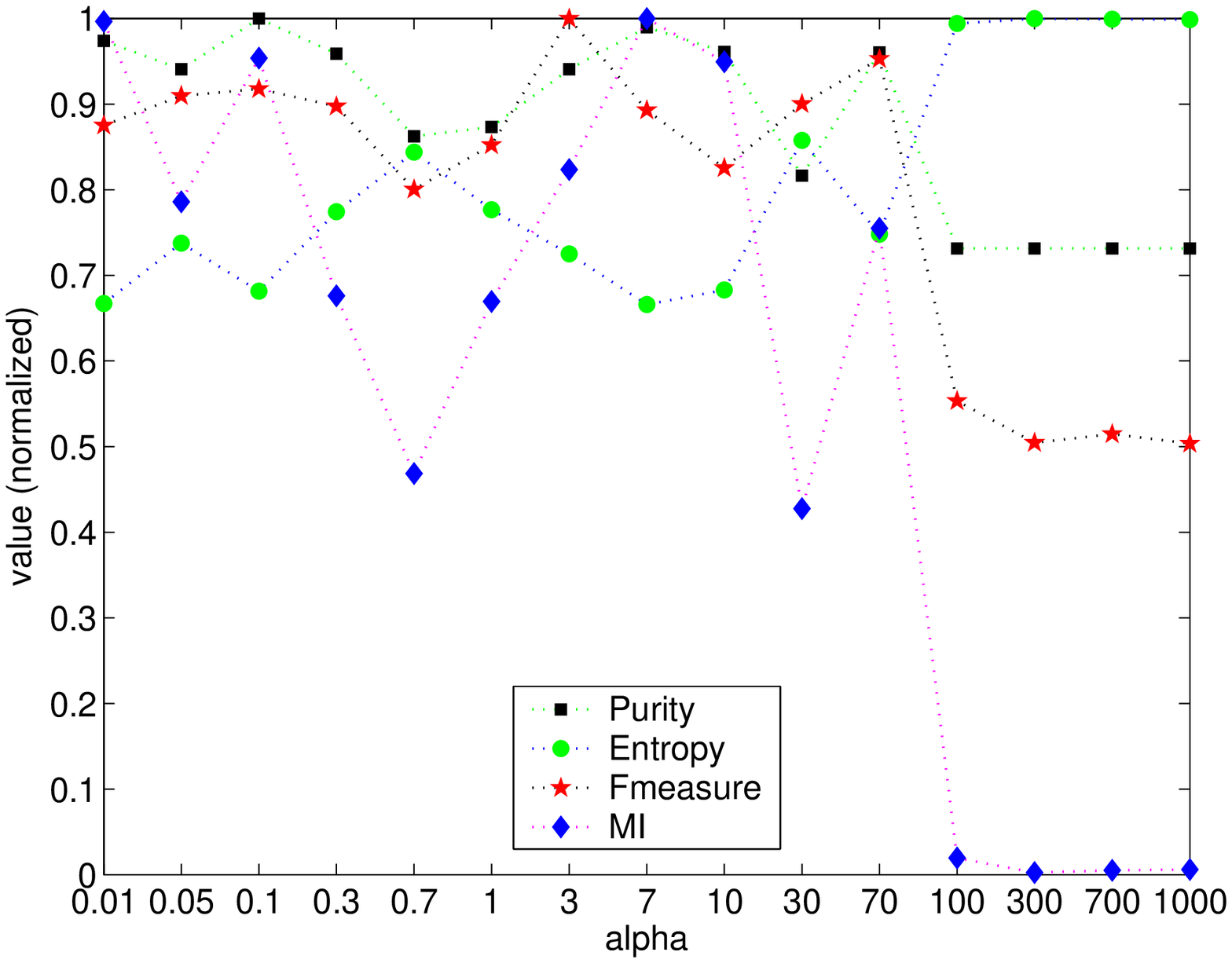}
   \label{fig9c}
  }
  \subfigure[AU-B($\alpha$), $\beta=1$]{
   \includegraphics[width=0.45\textwidth]{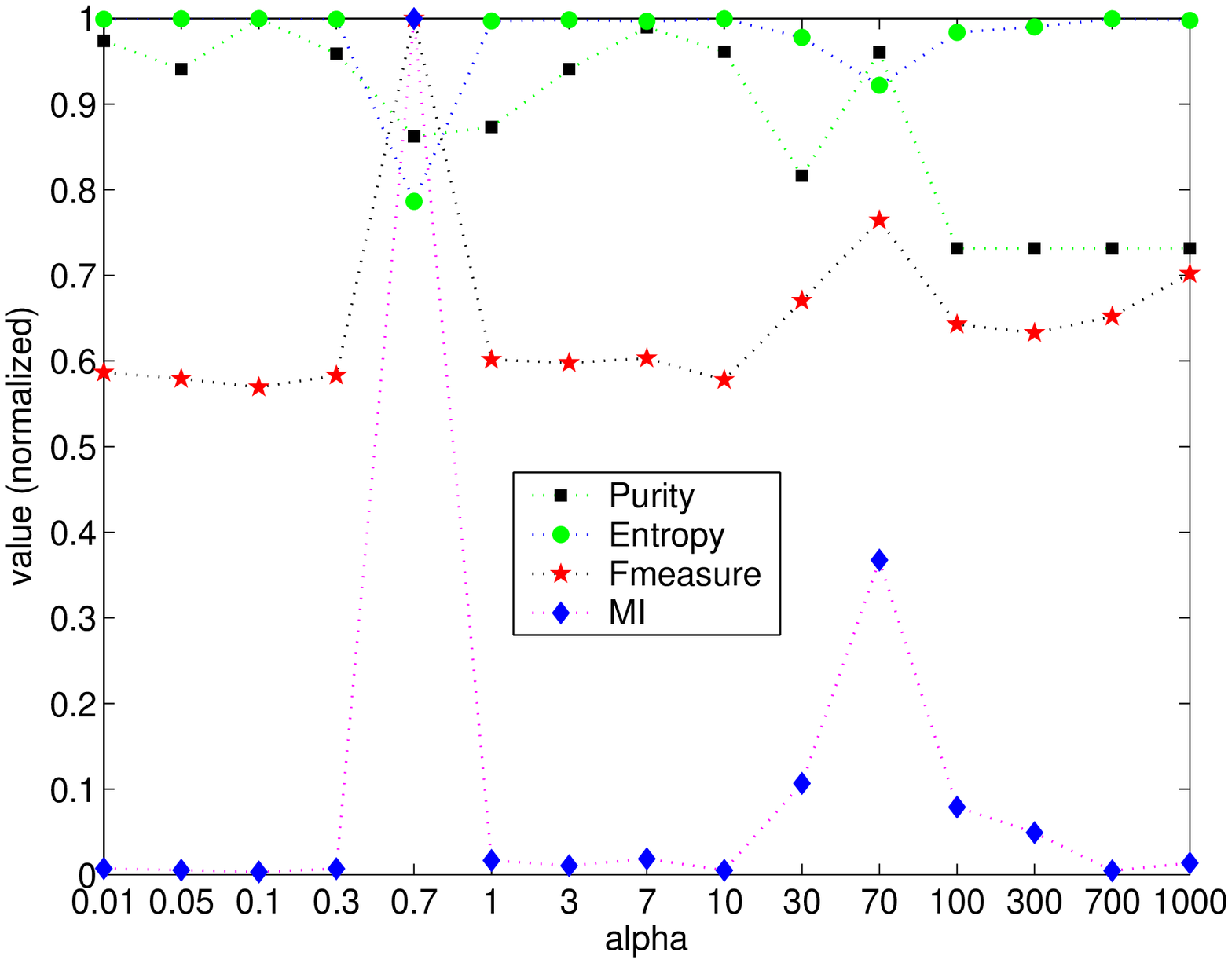}
   \label{fig9d}
  }
\\
  \subfigure[MU-B($\beta$), $\alpha=1$]{
   \includegraphics[width=0.45\textwidth]{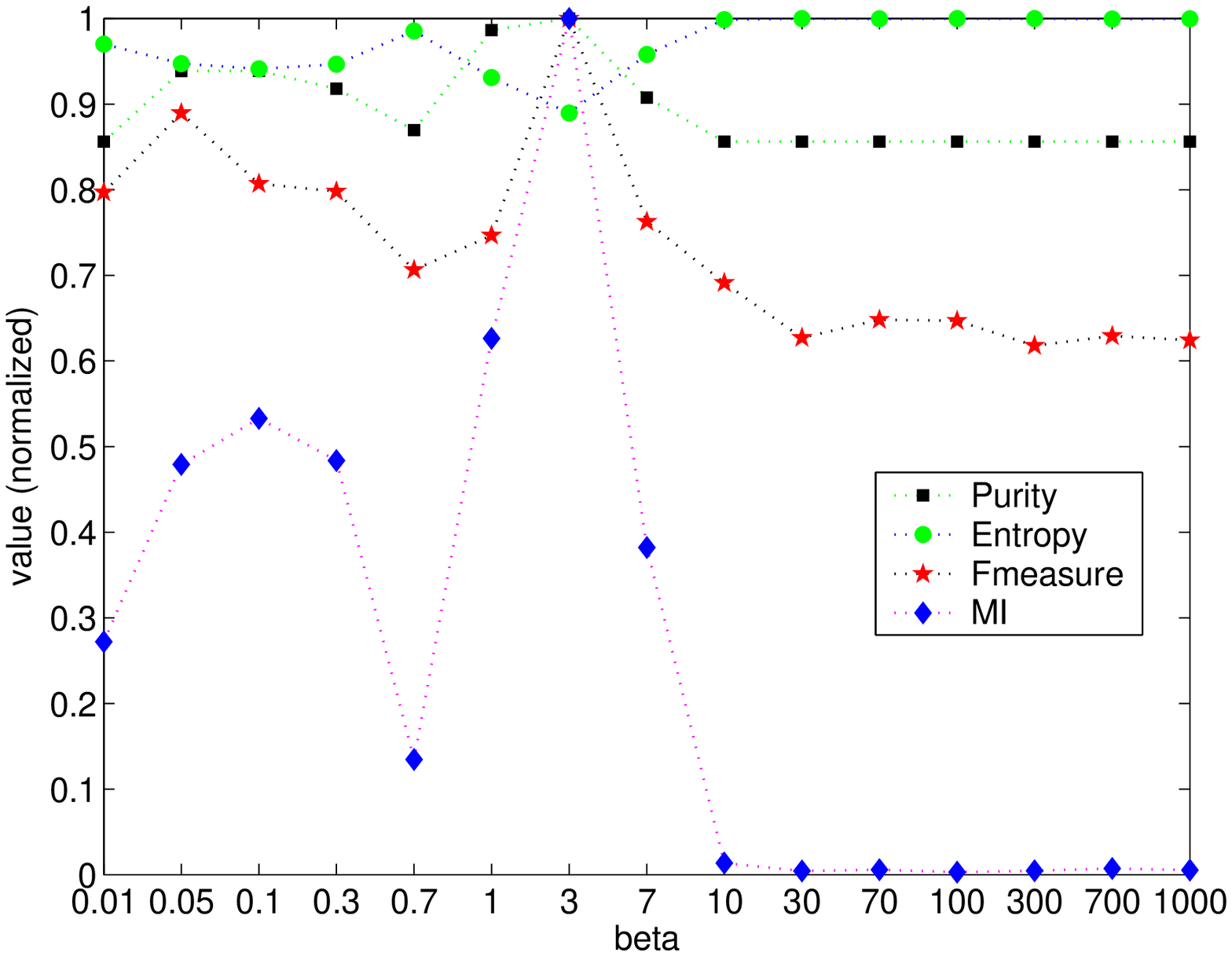}
   \label{fig9e}
  }
  \subfigure[AU-B($\beta$), $\alpha=1$]{
   \includegraphics[width=0.45\textwidth]{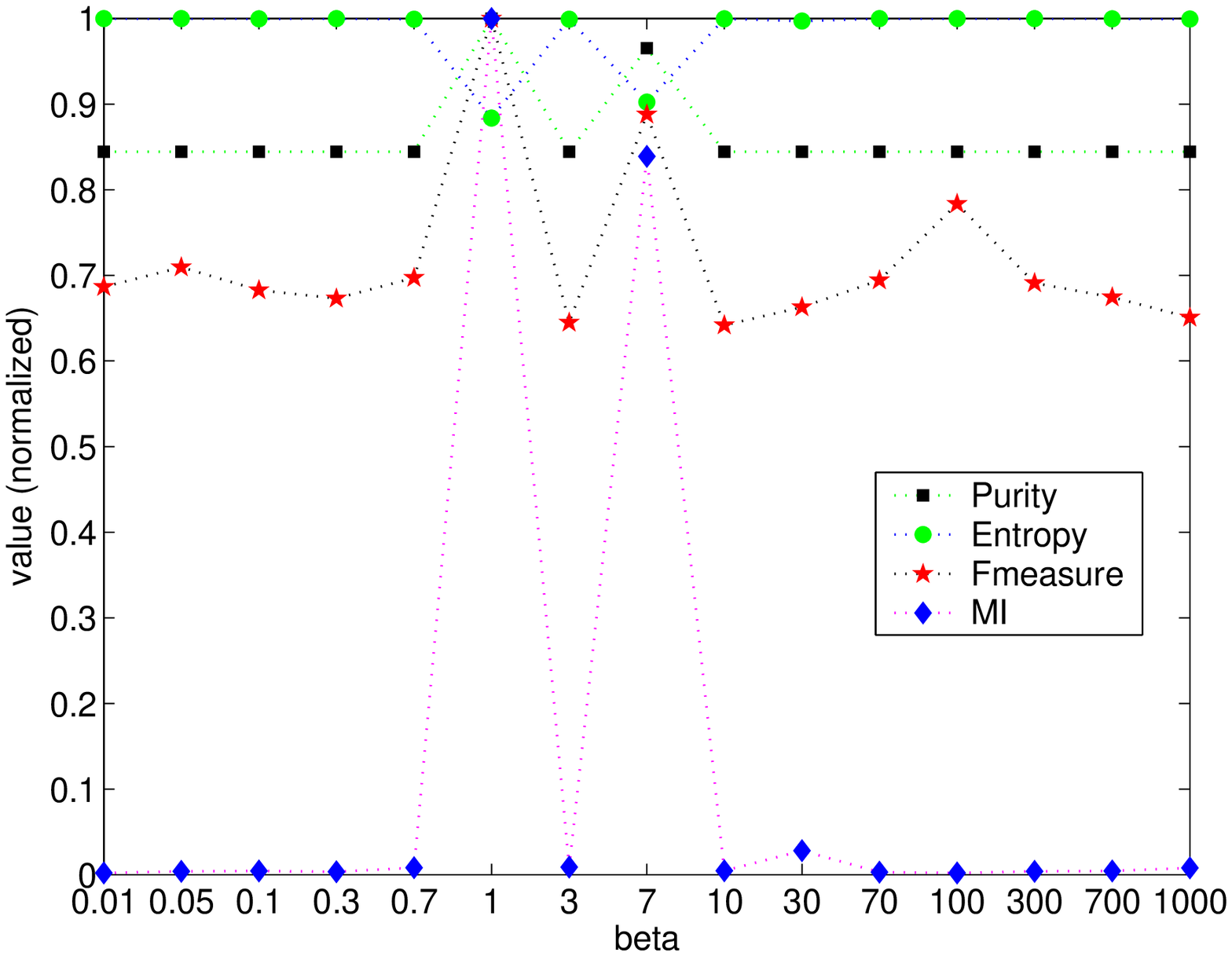}
   \label{fig9f}
  }
\caption{Clustering qualities as functions of $\alpha$ or $\beta$ for Reuters4.}
  \label{fig9}
 \end{center}
\end{figure}

\subsection{Times, \#iterations, and errors}

To evaluate computational performances of the algorithms, we measure average and maximum running times, average and maximum \#iterations, and average and maximum errors produced at the last iterations for 10 trials. Table \ref{table6}-\ref{table8} show the results.

As shown in the table \ref{table6}, LS generally is the fastest; however when MU-B or AU-B converge before reaching the maximum iteration (20 iterations), then these algorithms outperform LS. Our uni-orthogonal algorithms (MU-U and AU-U) seem to have comparable running times with LS. MU-B seems to be slower for smaller datasets and then performs better than MU-U and AU-U for larger datasets: Reuters10 and Reuters12. Since AU-B usually converges before reaching the maximum iteration (see table \ref{table7}), comparison can be done by using maximum running times where for Reuters4, Reuters6, Reuters10, and Reuters12 the data is available. As shown in table \ref{table7}, AU-B is the slowest to perform calculation per iteration. There are also abrupt changes in the running times for Reuters10 and Reuters12 for all algorithms which are unfortunate since as shown in table \ref{ch2:table3}, the sizes of the datasets are only slightly larger. Figure \ref{fig10} shows the average running times as the sizes of the datasets grow.

Average and maximum errors at the last iterations are shown in table \ref{table8}. Results for D-U and D-B are unsurprisingly really high as these algorithms do not minimize the objectives that are supposed to be minimized. Because only MU-U \& AU-U and MU-B \& AU-B pairs have the same objective each, we compare average errors for these pairs in figure \ref{fig11} for each dataset. There is no significant difference between MU-U \& AU-U in the average errors, but as shown in figure \ref{fig10}, MU-U has better average running times especially for larger datasets. And for MU-B \& AU-B, the differences in the average errors grow slightly as the size and classes of the datasets grow with significant differences occured at Reuters10 and Reuters12 where in these cases MU-B significantly outperformed AU-B.

\begin{table}
\renewcommand{\arraystretch}{1}
 \begin{center}
   \caption{Average and maximum running time.}
   \centering
   \small{
   \begin{tabular}{|l|l|r|r|r|r|r|r|r|}
   \hline
Data & Time & LS & D-U & D-B & MU-U & AU-U & MU-B & AU-B \\
\hline
Reuters2 & Av.  & 77.266 & 83.655 & 104.98 & 78.068 & 77.825 & 66.318 & 38.367 \\
         & Max. & 79.031 & 84.743 & 106.25 & 79.075 & 79.176 & 83.960 & 49.477 \\ \hline 
Reuters4 & Av.  & 108.84 & 119.42 & 152.77 & 109.04 & 109.12 & 119.46 & 86.745 \\
         & Max. & 109.39 & 119.55 & 153.17 & 109.20 & 109.28 & 119.72 & 271.40 \\ \hline 
Reuters6 & Av.  & 134.02 & 149.32 & 194.43 & 133.91 & 134.19 & 149.63 & 75.432 \\
         & Max. & 134.50 & 149.62 & 194.75 & 134.27 & 134.51 & 149.95 & 327.70 \\ \hline 
Reuters8 & Av.  & 158.37 & 173.43 & 228.59 & 153.53 & 155.03 & 173.00 & 56.464 \\
         & Max. & 181.58 & 175.71 & 235.54 & 155.15 & 159.19 & 174.05 & 59.021 \\ \hline 
Reuters10 & Av. & 834.69 & 892.91 & 911.34 & 874.18 & 914.93 & 859.31 & 601.57 \\
         & Max. & 1004.5 & 1141.2 & 1127.3 & 1137.5 & 1162.0 & 1059.0 & 2794.1 \\ \hline 
Reuters12 & Av. & 1249.2 & 1348.4 & 1440.1 & 1319.7 & 1335.6 & 1309.0 & 1602.4 \\
         & Max. & 1389.0 & 1590.4 & 1746.1 & 1565.7 & 1529.4 & 1506.7 & 4172.2 \\ 
\hline
  \end{tabular}}
  \label{table6}
 \end{center}
\end{table}

\begin{table}
\renewcommand{\arraystretch}{1}
 \begin{center}
   \caption{Average and maximum \#iteration.}
   \centering
   \small{
   \begin{tabular}{|l|l|r|r|r|r|r|r|r|}
   \hline
Data &\#iter. & LS & D-U & D-B & MU-U & AU-U & MU-B & AU-B \\
\hline
Reuters2 & Av.  & 20 & 20 & 20 & 20 & 20 & 16.2 & 4.9 \\
         & Max. & 20 & 20 & 20 & 20 & 20 & 20 & 6 \\ \hline 
Reuters4 & Av.  & 20 & 20 & 20 & 20 & 20 & 20 & 7.2 \\
         & Max. & 20 & 20 & 20 & 20 & 20 & 20 & 20 \\ \hline 
Reuters6 & Av.  & 20 & 20 & 20 & 20 & 20 & 20 & 5.5 \\
         & Max. & 20 & 20 & 20 & 20 & 20 & 20 & 20 \\ \hline 
Reuters8 & Av.  & 20 & 20 & 20 & 20 & 20 & 20 & 4 \\
         & Max. & 20 & 20 & 20 & 20 & 20 & 20 & 4 \\ \hline 
Reuters10 & Av. & 20 & 20 & 20 & 20 & 20 & 20 & 5.6 \\
         & Max. & 20 & 20 & 20 & 20 & 20 & 20 & 20 \\ \hline 
Reuters12 & Av. & 20 & 20 & 20 & 20 & 20 & 20 & 8.8 \\
         & Max. & 20 & 20 & 20 & 20 & 20 & 20 & 20 \\ 
\hline
  \end{tabular}}
  \label{table7}
 \end{center}
\end{table}

\begin{table}
\renewcommand{\arraystretch}{1}
 \begin{center}
   \caption{Average and maximum errors at the last iteration.}
   \centering
   \small{
   \begin{tabular}{|l|l|r|r|r|r|r|r|r|}
   \hline
Data &\#iter. & LS & D-U & D-B & MU-U & AU-U & MU-B & AU-B \\
\hline
Reuters2 & Av.  & 1.3763 & 3435.6 & 3626.5 & 1.4106 & 1.4138 & 1.7955 & 1.8021  \\
         & Max. & 1.3854 & 3587.2 & 3867.4 & 1.4201 & 1.4230 & 1.8022 & 1.8025  \\ \hline 
Reuters4 & Av.  & 1.4791 & 9152.8 & 8689.0 & 1.5299 & 1.5310 & 2.0708 & 2.0962 \\
         & Max. & 1.4855 & 9474.9 & 9297.9 & 1.5408 & 1.5402 & 2.0880 & 2.1028 \\ \hline 
Reuters6 & Av.  & 1.5229 & 17135 & 15823 & 1.5844 & 1.5878 & 2.2627 & 2.2921 \\
         & Max. & 1.5301 & 17971 & 16955 & 1.5884 & 1.5952 & 2.2758 & 2.2998 \\ \hline 
Reuters8 & Av.  & 1.5434 & 25913 & 22893 & 1.6215 & 1.6171 & 2.3863 & 2.4421 \\
         & Max. & 1.5473 & 27462 & 25553 & 1.6342 & 1.6262 & 2.3993 & 2.4422 \\ \hline 
Reuters10 & Av. & 1.5696 & 34154 & 30518 & 1.6533 & 1.6533 & 1.8836 & 2.5673 \\
         & Max. & 1.5801 & 35236 & 35152 & 1.6662 & 1.6618 & 1.9529 & 2.5718 \\ \hline 
Reuters12 & Av. & 1.5727 & 42739 & 37038 & 1.6620 & 1.6621 & 1.8860 & 2.6551  \\
         & Max. & 1.5815 & 44325 & 41940 & 1.6705 & 1.6713 & 1.9193 & 2.6697 \\ 
\hline
  \end{tabular}}
  \label{table8}
 \end{center}
\end{table}

\begin{figure}
\begin{center}
\includegraphics[width=0.6\textwidth]{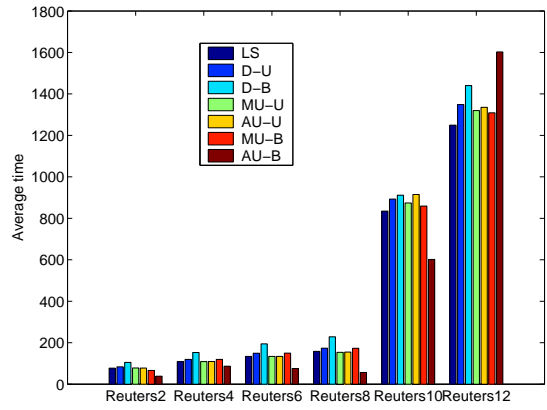}
\caption{Average running time comparison as the datasets grow.}
\label{fig10}
\end{center}
\end{figure}

\begin{figure}
 \begin{center}
  \subfigure[MU-U and AU-U.]{
   \includegraphics[width=0.45\textwidth]{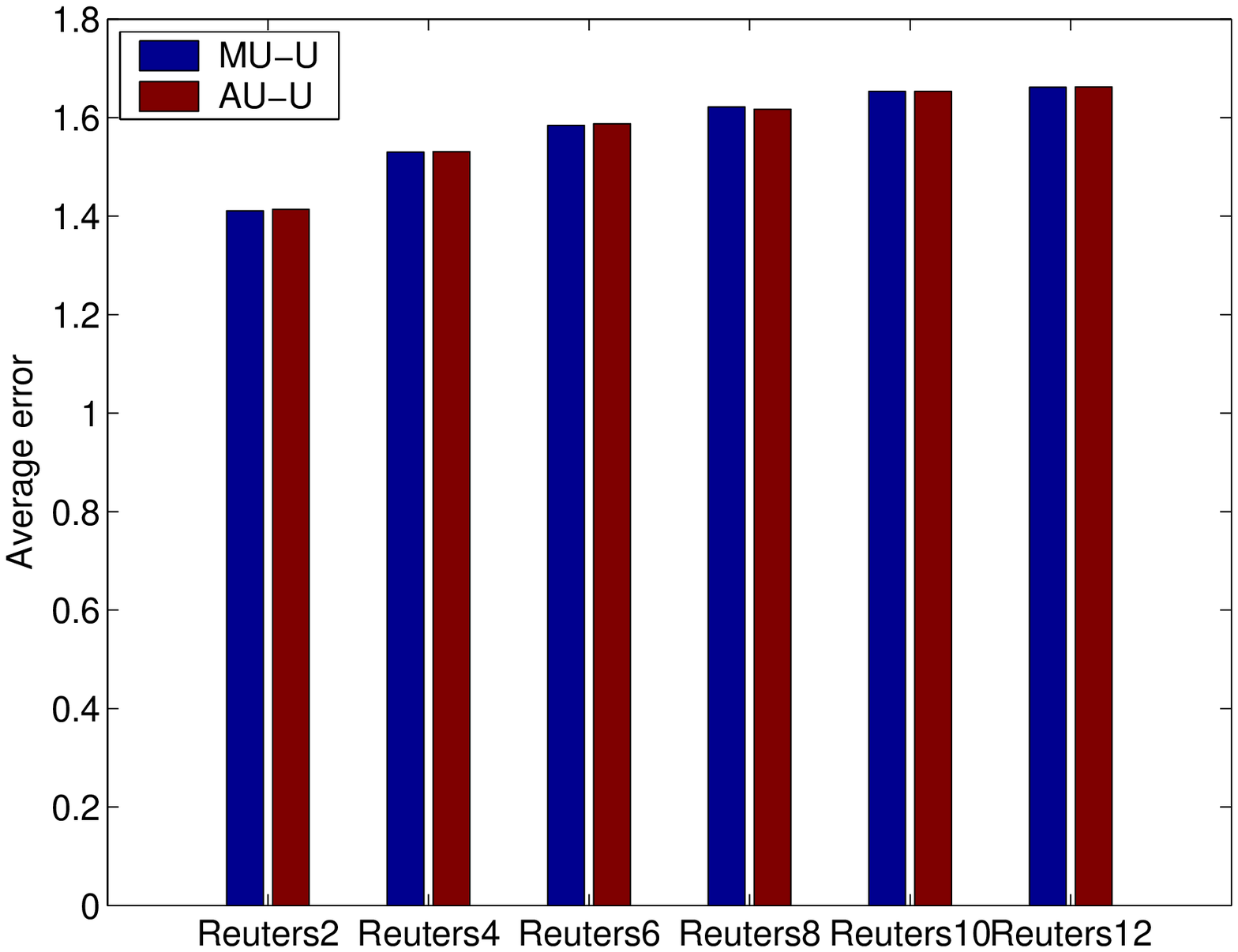}
   \label{fig11a}
  }
  \subfigure[MU-B and AU-B.]{
   \includegraphics[width=0.45\textwidth]{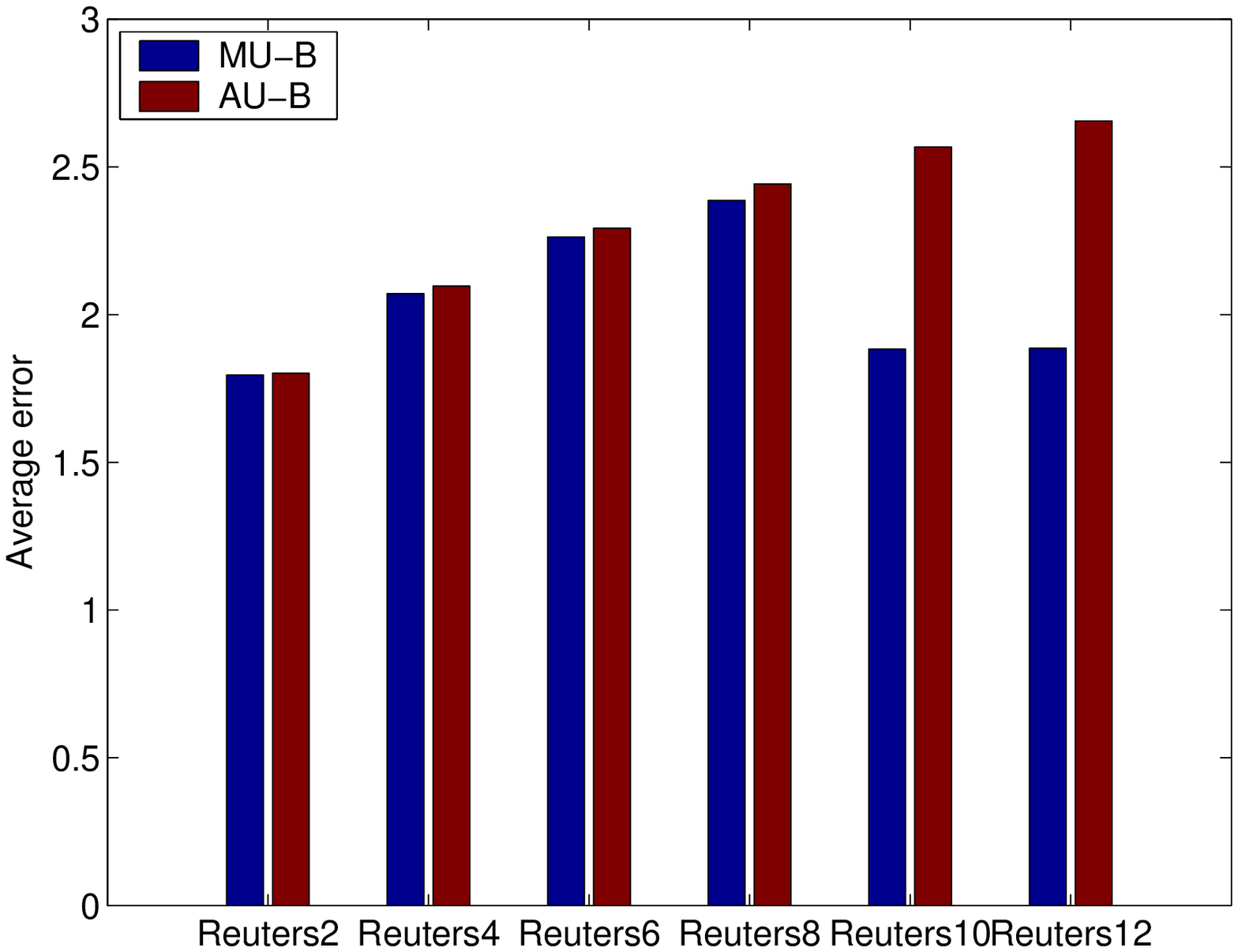}
   \label{fig11b}
  }
  \caption{Average errors comparison as the datasets grow.}
  \label{fig11}
 \end{center}
\end{figure}

\subsection{Document clustering}

The results of document clustering are shown in table \ref{table9}--\ref{table12}. In average, MU-U gives the best performances in all metrics especially for datasets with small \#clusters. Then followed by LS, AU-U, and D-U with small margins. LS seems to be better for datasets with large \#clusters. Generally, MU-U, LS, AU-U and D-U can give consistent results for varied \#clusters, but unfortunately this is not the case for D-B, MU-B and AU-B which are all BNMtF algorithms. AU-B especially seems to offer only slightly better clustering than random results.

\begin{table}
 \begin{center}
   \caption{Average mutual information over 10 trials (document clustering).}
   \centering
   \footnotesize{
   \begin{tabular}{|l|r|r|r|r|r|r|r|}
   \hline
Data & LS & D-U & D-B & MU-U & AU-U & MU-B & AU-B \\
\hline
Reuters2 & 0.40392 & 0.42487 & 0.36560 & $\mathbf{0.47507}$ & 0.42150 & 0.057799 & 0.00087646 \\ 
Reuters4 & 0.62879 & 0.61723 & 0.48007 & $\mathbf{0.65080}$ & 0.63640 & 0.32142 & 0.072621 \\ 
Reuters6 & 0.79459 & 0.81831 & 0.52498 & 0.81811 & $\mathbf{0.82425}$ & 0.37924 & 0.078201 \\ 
Reuters8 & 0.92285 & 0.90260 & 0.54534 & $\mathbf{0.94165}$ & 0.92720 & 0.48435 & 0.013518 \\ 
Reuters10 & $\mathbf{1.0415}$ & 1.0275 & 0.62125 & 1.0063 & 1.0138 & 0.50980 & 0.072014 \\ 
Reuters12 & $\mathbf{1.1326}$ & 1.0865 & 0.58469 & 1.1195 & 1.0821 & 0.47697 & 0.16389 \\ \hline
Average & 0.82071 & 0.81283 & 0.52032 & $\mathbf{0.83523}$ & 0.81754 & 0.37160 & 0.066853 \\
\hline
  \end{tabular}}
  \label{table9}
 \end{center}
\end{table}

\begin{table}
 \begin{center}
   \caption{Average entropy over 10 trials (document clustering).}
   \centering
   \footnotesize{
   \begin{tabular}{|l|r|r|r|r|r|r|r|}
   \hline
Data & LS & D-U & D-B & MU-U & AU-U & MU-B & AU-B \\
\hline
Reuters2 & 0.54193 & 0.52098 & 0.58025 & $\mathbf{0.47078}$ & 0.52435 & 0.88805 & 0.94498 \\ 
Reuters4 &  0.40202 & 0.40780 & 0.47638 & $\mathbf{0.39102}$ & 0.39822 & 0.55571 & 0.68011 \\ 
Reuters6 &  0.38391 & 0.37473 & 0.48821 & 0.37481 & $\mathbf{0.37243}$ & 0.54459 & 0.66105 \\ 
Reuters8 & 0.35568 & 0.36242 & 0.48151 & $\mathbf{0.34941}$ & 0.35423 & 0.50184 & 0.65879 \\ 
Reuters10 & $\mathbf{0.33601}$ & 0.34023 & 0.46253 & 0.34661 & 0.34434 & 0.49608 & 0.62786 \\ 
Reuters12 & $\mathbf{0.31953}$ & 0.33239 & 0.47236 & 0.32319 & 0.33362 & 0.50241 & 0.58974 \\ 
\hline
Average & 0.38985 & 0.389760 & 0.49354 & $\mathbf{0.37597}$ & 0.38787 & 0.58145 & 0.69375 \\ 
\hline
  \end{tabular}}
  \label{table10}
 \end{center}
\end{table}

\begin{table}
 \begin{center}
   \caption{Average purity over 10 trials (document clustering).}
   \centering
   \footnotesize{
   \begin{tabular}{|l|r|r|r|r|r|r|r|}
   \hline
Data & LS & D-U & D-B & MU-U & AU-U & MU-B & AU-B \\
\hline
Reuters2 & 0.82154 & 0.83599 & 0.80452 & $\mathbf{0.85089}$ & 0.82507 & 0.66102 & 0.63612 \\ 
Reuters4 &  0.79417 & 0.78023 & 0.73778 & $\mathbf{0.80400}$ & 0.79704 & 0.70119 & 0.59657 \\ 
Reuters6 &  0.74510 & $\mathbf{0.75158}$ & 0.68844 & 0.74868 & 0.75069 & 0.66433 & 0.54569 \\ 
Reuters8 &  $\mathbf{0.74906}$ & 0.73982 & 0.66536 & 0.74869 & 0.73987 & 0.65033 & 0.50680 \\ 
Reuters10 & 0.73120 & $\mathbf{0.73762}$ & 0.64845 & 0.72813 & 0.73330 & 0.63194 & 0.50639 \\ 
Reuters12 & 0.73877 & 0.72719 & 0.62223 & $\mathbf{0.74127}$ & 0.72340 & 0.60118 & 0.52019 \\ 
\hline
Average & 0.76331 & 0.76207 & 0.69446 & $\mathbf{0.77028}$ & 0.76156 & 0.65166 & 0.55196 \\ 
\hline
  \end{tabular}}
  \label{table11}
 \end{center}
\end{table}

\begin{table}
 \begin{center}
   \caption{Average Fmeasure over 10 trials (document clustering).}
   \centering
   \footnotesize{
   \begin{tabular}{|l|r|r|r|r|r|r|r|}
   \hline
Data & LS & D-U & D-B & MU-U & AU-U & MU-B & AU-B \\
\hline
Reuters2 & 0.81904 & 0.83234 & 0.79163 & $\mathbf{0.84823}$ & 0.82241 & 0.58237 & 0.50399 \\ 
Reuters4 & 0.56154 & 0.53754 & 0.44352 & $\mathbf{0.57989}$ & 0.54267 & 0.36917 & 0.24585 \\ 
Reuters6 & 0.46225 & 0.47714 & 0.33910 & $\mathbf{0.48444}$ & 0.47270 & 0.26372 & 0.17171 \\ 
Reuters8 & 0.40408 & 0.40554 & 0.25052 & 0.41822 & $\mathbf{0.42996}$ & 0.23904 & 0.10869 \\ 
Reuters10 & 0.38001 & $\mathbf{0.38041}$ & 0.23309 & 0.36923 & 0.35947 & 0.19552 & 0.094912 \\ 
Reuters12 & 0.35671 & $\mathbf{0.35811}$ & 0.17387 & 0.35214 & 0.34435 & 0.16401 & 0.099949 \\ 
\hline
Average & 0.49727 & 0.49851 & 0.37196 & $\mathbf{0.50869}$ & 0.49526 & 0.30231 & 0.20418 \\ 
\hline
  \end{tabular}}
  \label{table12}
 \end{center}
\end{table}

\subsection{Word clustering}

In some cases, the ability of clustering methods to simultaneously group similar documents with related words (co-clustering) can become an added value. And because the original BNMtF is designed to have this ability, we will also investigate the quality of word clustering (in the context of co-clustering) produced by all algorithms. Since word clustering has no reference class, we adopt idea from \cite{Ding1} where the authors proposed to create reference classes by using word frequencies: each word is assigned to a class where the word appears with the highest frequency. Table \ref{table13}--\ref{table16} show the results. As shown, D-U has the best overall results followed by LS, MU-U and AU-U by small margins. MU-U is especially good for small \#clusters and LS is good for large \#clusters. However, all BNMtF algorithms, D-B, MU-B, and AU-B, which are designed to accomodate co-clustering task, seem to have poor results. These results are in accord with document clustering results where BNMtFs also perform poorly.

\begin{table}
 \begin{center}
   \caption{Average mutual information over 10 trials (word clustering).}
   \centering
   \footnotesize{
   \begin{tabular}{|l|r|r|r|r|r|r|r|}
   \hline
Data & LS & D-U & D-B & MU-U & AU-U & MU-B & AU-B \\
\hline
Reuters2 & 0.15715 & 0.16609 & 0.12966 & $\mathbf{0.17351}$ & 0.14978 & 0.013995 & 0.00029807 \\ 
Reuters4 &  0.42558 & 0.39193 & 0.21495 & $\mathbf{0.42619}$ & 0.41663 & 0.11812 & 0.026943 \\ 
Reuters6 &  0.54112 & $\mathbf{0.57472}$ &	0.26971 & 0.54239 & 0.54828 & 0.12460 & 0.035309 \\ 
Reuters8 &  0.63022 & 0.63368 & 0.29277 & 0.64699 & $\mathbf{0.65774}$ & 0.15692 & 0.0037071 \\ 
Reuters10 &  0.70386 & $\mathbf{0.73345}$ & 0.33046 & 0.66262 & 0.68367 & 0.025320 & 0.029618 \\ 
Reuters12 & 0.80111 & 0.77959 & 0.28412 & 0.76128 & 0.73517 & 0.013483 & 0.073478 \\ \hline
Average &  0.54317 & $\mathbf{0.54658}$ & 0.25361 & 0.53549 & 0.53188 & 0.075407 & 0.028226 \\
\hline
  \end{tabular}}
  \label{table13}
 \end{center}
\end{table}

\begin{table}
 \begin{center}
   \caption{Average entropy over 10 trials (word clustering).}
   \centering
   \footnotesize{
   \begin{tabular}{|l|r|r|r|r|r|r|r|}
   \hline
Data & LS & D-U & D-B & MU-U & AU-U & MU-B & AU-B \\
\hline
Reuters2 & 0.76778 & 0.75884 & 0.79527 & $\mathbf{0.75142}$ & 0.77515 & 0.91094 & 0.92463 \\ 
Reuters4 & 0.62965 & 0.64647 & 0.73496 & $\mathbf{0.62934}$ & 0.63412 & 0.78338 & 0.82897 \\ 
Reuters6 & 0.56184 & $\mathbf{0.54884}$ & 0.66683 & 0.56134 & 0.55906 & 0.72297 & 0.75751 \\ 
Reuters8 & 0.52006 & 0.51891 & 0.63255 & 0.51447 & $\mathbf{0.51089}$ & 0.67783 & 0.72890 \\ 
Reuters10 & 0.50612 & $\mathbf{0.49721}$ & 0.61852 & 0.51853 & 0.51220 & 0.71038 & 0.70909 \\ 
Reuters12 & $\mathbf{0.48211}$ & 0.48811 &	0.62632 & 0.49322 & 0.50050 & 0.70181 & 0.68507 \\ \hline
Average & 0.57792 & $\mathbf{0.57640}$ & 0.67908 & 0.57806 & 0.58199 & 0.75122 & 0.77236 \\
\hline
  \end{tabular}}
  \label{table14}
 \end{center}
\end{table}

\begin{table}
 \begin{center}
   \caption{Average purity over 10 trials (word clustering).}
   \centering
   \footnotesize{
   \begin{tabular}{|l|r|r|r|r|r|r|r|}
   \hline
Data & LS & D-U & D-B & MU-U & AU-U & MU-B & AU-B \\
\hline
Reuters2 & 0.76987 & 0.77082 & 0.75378 & $\mathbf{0.77730}$ & 0.76021 & 0.67006 & 0.65988 \\ 
Reuters4 & 0.64400 & 0.62881 & 0.60566 & $\mathbf{0.64676}$ & 0.64184 & 0.55808 & 0.53116 \\ 
Reuters6 & 0.59830 & $\mathbf{0.61733}$ & 0.55949 & 0.59763 & 0.59103 & 0.52966 & 0.49661 \\ 
Reuters8 & $\mathbf{0.59560}$ & 0.58935 & 0.54296 & 0.59179 & 0.58770 & 0.50933 & 0.46499 \\ 
Reuters10 & 0.58123 & $\mathbf{0.60236}$ & 0.51576 & 0.57045 & 0.58724 & 0.44765 & 0.45395 \\ 
Reuters12 &  $\mathbf{0.60208}$ & 0.59563 & 0.49555 & 0.58628 & 0.56846 & 0.43611 & 0.44882 \\ \hline
Average & 0.63185 & $\mathbf{0.63405}$ & 0.57887 & 0.62837 & 0.62274 & 0.52515 & 0.50923 \\
\hline
  \end{tabular}}
  \label{table15}
 \end{center}
\end{table}

\begin{table}
 \begin{center}
   \caption{Average Fmeasure over 10 trials (word clustering).}
   \centering
   \footnotesize{
   \begin{tabular}{|l|r|r|r|r|r|r|r|}
   \hline
Data & LS & D-U & D-B & MU-U & AU-U & MU-B & AU-B \\
\hline
Reuters2 & 0.59287 & 0.59471 & 0.58733 & $\mathbf{0.59696}$ & 0.59427 & 0.52628 & 0.49976 \\ 
Reuters4 & 0.46891 & 0.43469 & 0.36397 & $\mathbf{0.48118}$ & 0.46180 & 0.32520 & 0.27101 \\ 
Reuters6 & 0.37490 & 0.38365 & 0.27356 & $\mathbf{0.38648}$ & 0.38026 & 0.21620 & 0.17572 \\ 
Reuters8 & 0.32488 & 0.32674 & 0.20820 & 0.33527 & $\mathbf{0.34251}$ & 0.17127 & 0.12565 \\ 
Reuters10 & 0.29864 & $\mathbf{0.30768}$ & 0.18626 & 0.28930 & 0.28573 & 0.10700 & 0.10545 \\ 
Reuters12 & $\mathbf{0.29116}$ & 0.29072 & 0.14255 & 0.27525 & 0.27380 & 0.088517 & 0.095880 \\ 
\hline
Average & 0.39189 & 0.38970 & 0.29365 & $\mathbf{0.39407}$ & 0.38973 & 0.23908 & 0.21224 \\
\hline
  \end{tabular}}
  \label{table16}
 \end{center}
\end{table}

\section{Conclusions}\label{conclusion}

We have presented a convergent algorithm for BNMtF based on a technique presented in our previous work. The convergence property of the algorithm is proven theoretically and its nonincreasing property is investigated numerically. As shown in the experimental results, the algorithm preserves the nonincreasing property even in the case where the regularization parameters are large. We then investigate some aspects of the algorithm like its behaviour under different values of regularization parameters, minimization slopes, computational times, and error per iterations. We also investigate the document/word clustering capabilities of the algorithm.

\begin{acknowledgements}
We are very grateful to the referees for their most interesting comments and suggestions.
\end{acknowledgements}

\bibliographystyle{spmpsci}      
\bibliography{paper}

\clearpage

\appendix

\section{Convergence analysis} \label{appendix}

From a result in convergence analysis of block coordinate descent method \cite{CJLin2,CJLin,Grippo}, algorithm \ref{algorithm8} has a convergence guarantee if the following conditions are satisfied:
\begin{enumerate}
\item sequence $J\big(\mathbf{B}^{(k)}, \mathbf{S}^{(k)}, \mathbf{C}^{(k)}\big)$ has nonincreasing property,
\item any limit point of sequence $\big\{\mathbf{B}^{(k)}, \mathbf{S}^{(k)}, \mathbf{C}^{(k)}\big\}$ generated by algorithm \ref{algorithm8} is a stationary point, and
\item sequence $\big\{\mathbf{B}^{(k)}, \mathbf{S}^{(k)}, \mathbf{C}^{(k)}\big\}$ has at least one limit point.
\end{enumerate}

Because algorithm \ref{algorithm8} uses the alternating strategy, sequences $J\big(\mathbf{B}^{(k)}\big)$, $J\big(\mathbf{C}^{(k)}\big)$, and $J\big(\mathbf{S}^{(k)}\big)$ can be analyzed separately \cite{Lee2,CJLin2}. And because update rule for $\mathbf{B}^{(k)}$ in eq.~\ref{eq98} is similar to the update rule for $\mathbf{C}^{(k)}$ in eq.~\ref{eq99}, it suffices to prove nonincreasing property of one of them.

\subsection{Nonincreasing property of $J\big(\mathbf{B}^{(k)}\big)$}

By using auxiliary function approach \cite{Lee2}, nonincreasing property of $J\big(\mathbf{B}^{(k)}\big)$ can be proven if the following statement is true:
\begin{equation*}
J\big(\mathbf{B}^{(k+1)}\big) = \; G\big(\mathbf{B}^{(k+1)},\mathbf{B}^{(k+1)}\big) \le G\big(\mathbf{B}^{(k+1)},\mathbf{B}^{(k)}\big) \le G\big(\mathbf{B}^{(k)},\mathbf{B}^{(k)}\big) = J\big(\mathbf{B}^{(k)}\big). 
\end{equation*} To define $G$, let us rearrange $\mathbf{B}$ into:
\begin{equation*}
\mathfrak{B}^T \equiv
\begin{bmatrix}
\mathfrak{b}_1^T & & & \\
& \mathfrak{b}_2^T & & \\
& & \ddots & \\
& & & \mathfrak{b}_M^T
\end{bmatrix}\in\mathbb{R}_+^{MP\times M},
\end{equation*}
where $\mathfrak{b}_m$ is the $m$-th row of $\mathbf{B}$. And also let us define:
\begin{equation*}
\nabla_{\mathfrak{B}^T}\mathfrak{J}\big(\mathfrak{B}^{(k)T}\big) \equiv
\begin{bmatrix}
\nabla_{\mathbf{B}}\mathfrak{J}\big(\mathbf{B}^{(k)}\big)_1^T & & & \\
& \nabla_{\mathbf{B}}\mathfrak{J}\big(\mathbf{B}^{(k)}\big)_2^T & & \\
& & \ddots & \\
& & & \nabla_{\mathbf{B}}\mathfrak{J}\big(\mathbf{B}^{(k)}\big)_M^T
\end{bmatrix}\in\mathbb{R}_+^{MP\times M},
\end{equation*}
where $\nabla_{\mathbf{B}}\mathfrak{J}\big(\mathbf{B}^{(k)}\big)_m$ is the $m$-th row of $\nabla_{\mathbf{B}}J(\mathbf{B}^{(k)})$. Then we define:
\begin{equation*}
\mathbf{D} \equiv \mathrm{diag}\;\big(\mathbf{D}^1,\ldots,\mathbf{D}^M\big)\in\mathbb{R}_+^{MP\times MP},
\end{equation*}
where $\mathbf{D}^m$ is a diagonal matrix with its diagonal entries defined as:
\begin{equation*}
d_{pp}^m \equiv \left\{
 \begin{array}{ll}
   \frac{\big( \mathbf{\bar{B}}^{(k)}\mathbf{S}^{(k)}\mathbf{C}^{(k)}\mathbf{C}^{(k)T}\mathbf{S}^{(k)T} + \beta\mathbf{\bar{B}}^{(k)}\mathbf{\bar{B}}^{(k)T}\mathbf{\bar{B}}^{(k)} \big)_{mp}+\delta_{\mathbf{B}}^{(k)}}{\bar{b}_{mp}^{(k)}} & \mathrm{if}\;\; p\in \mathcal{I}_m \\
   \star & \mathrm{if}\;\; p\notin \mathcal{I}_m
 \end{array} \right.
\end{equation*}
with
\begin{align*}
\mathcal{I}_m \equiv \big\{p|&b_{mp}^{(k)}>0,\;\nabla_{\mathbf{B}}J\big(\mathbf{B}^{(k)}\big)_{mp}\ne 0,\;\mathrm{or} \\
&b_{mp}^{(k)}=0,\;\nabla_{\mathbf{B}}J\big(\mathbf{B}^{(k)}\big)_{mp} < 0\big\}
\end{align*}
denotes the set of non-KKT indices in $m$-th row of $\mathbf{B}^{(k)}$, and the symbol $\star$ is defined so that $\star \equiv 0$ and $\star^{-1} \equiv 0$.

Then, the auxiliary function $\mathfrak{G}$ can be defined as:
\begin{equation}
\mathfrak{G}\big(\mathfrak{B}^T,\mathfrak{B}^{(k)T}\big) \equiv \mathfrak{J}\big(\mathfrak{B}^{(k)T}\big) + \mathrm{tr}\;\big\{\big(\mathfrak{B}-\mathfrak{B}^{(k)}\big)\nabla_{\mathfrak{B}^T}\mathfrak{J}\big(\mathfrak{B}^{(k)T}\big)\big\} + \frac{1}{2}\mathrm{tr}\;\big\{\big(\mathfrak{B}-\mathfrak{B}^{(k)}\big)\mathbf{D}\big(\mathfrak{B}-\mathfrak{B}^{(k)}\big)^T\big\}. \label{eq109}
\end{equation}
Note that $\mathfrak{J}$ and $\mathfrak{G}$ are equivalent to $J$ and $G$ with $\mathbf{B}$ is rearranged into $\mathfrak{B}^T$, and other variables are reordered accordingly. And:
\begin{equation*}
\nabla_{\mathfrak{B}^T}\mathfrak{G}\big(\mathfrak{B}^T,\mathfrak{B}^{(k)T}\big)=\mathbf{D}\big(\mathfrak{B}-\mathfrak{B}^{(k)}\big)^T + \nabla_{\mathfrak{B}^T}\mathfrak{J}\big(\mathfrak{B}^{(k)T}\big).
\end{equation*}
By definition, $\mathbf{D}$ is positive definite for all $\mathbf{B}^{(k)}$ not satisfy the KKT conditions, so $\mathfrak{G}\big(\mathfrak{B}^T,\mathfrak{B}^{(k)T}\big)$ is a strict convex function, and consequently has a unique minimum.
\begin{align}
\mathbf{D}\big(\mathfrak{B}-\mathfrak{B}^{(k)}\big)^T + \nabla_{\mathfrak{B}^T}\mathfrak{J}\big(\mathfrak{B}^{(k)T}\big)=0, \label{eq110}\\
\mathfrak{B}^T = \mathfrak{B}^{(k)T} - \mathbf{D}^{-1}\nabla_{\mathfrak{B}^T}\mathfrak{J}\big(\mathfrak{B}^{(k)T}\big), \nonumber
\end{align}
which is exactly the update rule for $\mathbf{B}^{(k)}$.

By using the Taylor series expansion, $\mathfrak{J}\big(\mathfrak{B}^T\big)$ can also be written as:
\begin{align}
\mathfrak{J}\big(\mathfrak{B}^T\big) = &\;\mathfrak{J}\big(\mathfrak{B}^{(k)T}\big) + \mathrm{tr}\;\big\{\big(\mathfrak{B}-\mathfrak{B}^{(k)}\big)\nabla_{\mathfrak{B}^T}\mathfrak{J}\big(\mathfrak{B}^{(k)T}\big)\big\} + \nonumber\\
&\;\frac{1}{2}\mathrm{tr}\;\big\{\big(\mathfrak{B}-\mathfrak{B}^{(k)}\big)\nabla_{\mathbf{B}}^2\mathbf{J}\big(\mathbf{B}^{(k)}\big)\big(\mathfrak{B}-\mathfrak{B}^{(k)}\big)^T\big\} + \varepsilon_{\mathbf{B}}^{(k)}, \label{eq111}
\end{align}
where
\begin{align*}
\varepsilon_{\mathbf{B}}^{(k)} = &\frac{1}{6}\mathrm{tr}\;\big\{\big(\mathfrak{B}-\mathfrak{B}^{(k)}\big)\big( 6\beta\mathfrak{B}^{(k)T} \big) \big(\mathfrak{B}-\mathfrak{B}^{(k)}\big) \big(\mathfrak{B}-\mathfrak{B}^{(k)}\big)^T\big\} + \\
&\frac{1}{24}\mathrm{tr}\;\big\{\big(\mathfrak{B}-\mathfrak{B}^{(k)}\big)\big(\mathfrak{B}-\mathfrak{B}^{(k)}\big)^T\big( 6\beta\mathbf{I} \big) \big(\mathfrak{B}-\mathfrak{B}^{(k)}\big) \big(\mathfrak{B}-\mathfrak{B}^{(k)}\big)^T\big\}
\end{align*}
and
\begin{equation*}
\nabla_{\mathbf{B}}^2\mathbf{J}\big(\mathbf{B}^{(k)}\big) \equiv
\begin{bmatrix}
\nabla_{\mathbf{B}}^2 J \big(\mathbf{B}^{(k)}\big) & & \\
& \ddots & \\
& & \nabla_{\mathbf{B}}^2 J \big(\mathbf{B}^{(k)}\big)
\end{bmatrix}\in\mathbb{R}_+^{MP\times MP}
\end{equation*}
with $\nabla_{\mathbf{B}}^2 J \big(\mathbf{B}^{(k)}\big)$ components are arranged along its diagonal area (there are $M$ components).

To show the nonincreasing property of $J\big(\mathbf{B}^{(k)}\big)$, the following statements must be proven:
\begin{enumerate}
\item $\mathfrak{G}\big(\mathfrak{B}^T,\mathfrak{B}^T\big)=\mathfrak{J}\big(\mathfrak{B}^T\big)$,
\item $\mathfrak{G}\big(\mathfrak{B}^{(k)T},\mathfrak{B}^{(k)T}\big)=\mathfrak{J}\big(\mathfrak{B}^{(k)T}\big)$,
\item $\mathfrak{G}\big(\mathfrak{B}^T,\mathfrak{B}^T\big) \le \mathfrak{G}\big(\mathfrak{B}^T,\mathfrak{B}^{(k)T}\big)$, and
\item $\mathfrak{G}\big(\mathfrak{B}^T,\mathfrak{B}^{(k)T}\big) \le \mathfrak{G}\big(\mathfrak{B}^{(k)T},\mathfrak{B}^{(k)T}\big)$,
\end{enumerate} The first and second statements are obvious from the definition of $\mathfrak{G}$ in eq.~\ref{eq109}, the third and the fourth statements will be proven in theorem \ref{theorem19} and \ref{theorem20}, and the boundedness of $\mathbf{B}^{(k)}$, $\mathbf{C}^{(k)}$, and $\mathbf{S}^{(k)}$ will be proven in theorem \ref{theorem32}.


\begin{theorem} \label{theorem19}
Given sufficiently large $\delta_{\mathbf{B}}^{(k)}$ and the boundedness of $\mathbf{B}^{(k)}$, $\mathbf{C}^{(k)}$, and $\mathbf{S}^{(k)}$, then it can be shown that $\mathfrak{G}\big(\mathfrak{B}^T,\mathfrak{B}^T\big) \le \mathfrak{G}\big(\mathfrak{B}^T,\mathfrak{B}^{(k)T}\big)$. Moreover, if and only if $\mathbf{B}^{(k)}$ satisfies the KKT conditions, then the equality holds.
\end{theorem}
\begin{proof}
By substracting eq.~\ref{eq109} from eq.~\ref{eq111}, we get:
\begin{align}
\mathfrak{G}\big(\mathfrak{B}^T,\mathfrak{B}^{(k)T}\big)-\mathfrak{G}\big(\mathfrak{B}^T,\mathfrak{B}^T\big)&=\frac{1}{2}\,\mathrm{tr}\,\big\{\big(\mathfrak{B}-\mathfrak{B}^{(k)}\big)\big(\mathbf{D}-\nabla_{\mathbf{B}}^2\mathbf{J}\big(\mathbf{B}^{(k)}\big)\big)\big(\mathfrak{B}-\mathfrak{B}^{(k)}\big)^T\big\} - \mathbf{\varepsilon}_{\mathbf{B}}^{(k)} \nonumber\\
&=\frac{1}{2}\sum_{m=1}^M\left[\big(\mathfrak{b}_m-\mathfrak{b}_m^{(k)}\big)\big(\mathbf{D}^m-\nabla_{\mathbf{B}}^2 J \big(\mathbf{B}^{(k)}\big)\big)\big(\mathfrak{b}_m-\mathfrak{b}_m^{(k)}\big)^T\right] - \mathbf{\varepsilon}_{\mathbf{B}}^{(k)}. \label{eq112}
\end{align}

Let $\mathbf{v}_m^T = \mathfrak{b}_m - \mathfrak{b}_m^{(k)}$, then:
\begin{align*}
\mathbf{v}_m^T\big(\mathbf{D}^m-\nabla_{\mathbf{B}}^2J\big(\mathbf{B}^{(k)}\big)\big)\mathbf{v}_m &= \mathbf{v}_m^T\big(\mathbf{D}^m + \beta\mathbf{I} - \big( \mathbf{S}^{(k)}\mathbf{C}^{(k)}\mathbf{C}^{(k)T}\mathbf{S}^{(k)T} + 3\beta\mathbf{B}^{(k)T}\mathbf{B}^{(k)}\big)\big)\mathbf{v}_m \\
&= \mathbf{v}_m^T\big(\mathbf{\bar{D}}^m + \delta_{\mathbf{B}}^{(k)}\mathbf{\hat{D}}^m + \beta\mathbf{I} - \big( \mathbf{S}^{(k)}\mathbf{C}^{(k)}\mathbf{C}^{(k)T}\mathbf{S}^{(k)T} + 3\beta\mathbf{B}^{(k)T}\mathbf{B}^{(k)} \big)\big)\mathbf{v}_m,
\end{align*} where $\mathbf{\bar{D}}^m$ and $\delta_{\mathbf{B}}^{(k)}\mathbf{\hat{D}}^m$ are diagonal matrices that summed up to $\mathbf{D}^m$, with
\begin{align*}
\bar{d}_{pp}^m &\equiv \left\{
 \begin{array}{ll}
   \frac{\big( \mathbf{\bar{B}}^{(k)}\mathbf{S}^{(k)}\mathbf{C}^{(k)}\mathbf{C}^{(k)T}\mathbf{S}^{(k)T} + \beta\mathbf{\bar{B}}^{(k)}\mathbf{\bar{B}}^{(k)T}\mathbf{\bar{B}}^{(k)} \big)_{mp}}{\bar{b}_{mp}^{(k)}} & \mathrm{if}\;\; p\in \mathcal{I}_m \\
   \star & \mathrm{if}\;\; p\notin \mathcal{I}_m,
 \end{array} \right.
\text{and}\;
\hat{d}_{pp}^m &\equiv \left\{
 \begin{array}{ll}
   \frac{1}{\bar{b}_{mp}^{(k)}} & \mathrm{if}\;\; p\in \mathcal{I}_m \\
   \star & \mathrm{if}\;\; p\notin \mathcal{I}_m.
 \end{array} \right.
\end{align*}
Accordingly,
\begin{align}
\mathfrak{G}\big(\mathfrak{B}^T,\mathfrak{B}^{(k)T}\big)-\mathfrak{G}\big(\mathfrak{B}^T,\mathfrak{B}^T\big) = &\frac{1}{2} \sum_{m=1}^M \left\{\sum_{p=1}^P v_{mp}^2 \bar{d}_{pp}^m + \delta_{\mathbf{B}}^{(k)} \sum_{p=1}^P v_{mp}^2 \hat{d}_{pp}^m + \beta \sum_{p=1}^P v_{mp}^2 \right\} \nonumber \\
&- \frac{1}{2} \sum_{m=1}^M \mathbf{v}_m^T \big( \mathbf{S}^{(k)}\mathbf{C}^{(k)}\mathbf{C}^{(k)T}\mathbf{S}^{(k)T} + 3\beta\mathbf{B}^{(k)T}\mathbf{B}^{(k)} \big)\mathbf{v}_m - \varepsilon_{\mathbf{B}}^{(k)}. \label{eq113}
\end{align}
As shown, with the boundedness of $\mathbf{B}^{(k)}$, $\mathbf{C}^{(k)}$, and $\mathbf{S}^{(k)}$ and by sufficiently large $\delta_{\mathbf{B}}^{(k)}$, $\mathfrak{G}\big(\mathfrak{B}^T,\mathfrak{B}^T\big) \le \mathfrak{G}\big(\mathfrak{B}^T,\mathfrak{B}^{(k)T}\big)$ can be guaranteed. 

Next we prove the second statement of the theorem. By eq.~\ref{eq112}, if $\mathbf{B}^{(k)}$ satisfies the KKT conditions, then the equality will hold. And by eq.~\ref{eq113}, since $\delta_{\mathbf{B}}^{(k)}$ is a variable, the equality will hold if and only if $\mathbf{B} = \mathbf{B}^{(k)}$, which by the update rule in eq.~\ref{eq98} will happen if and only if $\mathbf{B}^{(k)}$ satisfies the KKT conditions. This completes the proof.
\end{proof}

\begin{theorem} \label{theorem20}
$\mathfrak{G}\big(\mathfrak{B}^T,\mathfrak{B}^{(k)T}\big) \le \mathfrak{G}\big(\mathfrak{B}^{(k)T},\mathfrak{B}^{(k)T}\big)$. Moreover if and only if $\mathbf{B}^{(k)}$ satisfies the KKT conditions in eqs.~\ref{eq91}, then the equality holds.
\end{theorem}
\begin{proof}
\begin{align*}
\mathfrak{G}\big(\mathfrak{B}^{(k)T},\mathfrak{B}^{(k)T}\big) - \mathfrak{G}\big(\mathfrak{B}^T,\mathfrak{B}^{(k)T}\big) =\;&-\mathrm{tr}\;\big\{\big(\mathfrak{B}-\mathfrak{B}^{(k)}\big)\nabla_{\mathfrak{B}}\mathfrak{J}\big(\mathfrak{B}^{(k)T}\big)\big\} \\
&- \frac{1}{2}\mathrm{tr}\;\big\{\big(\mathfrak{B}-\mathfrak{B}^{(k)}\big)\mathbf{D}\big(\mathfrak{B}-\mathfrak{B}^{(k)}\big)^T\big\}.
\end{align*}
Substituting eq.~\ref{eq110} into the above equation, we get:
\begin{equation*}
\mathfrak{G}\big(\mathfrak{B}^{(k)T},\mathfrak{B}^{(k)T}\big) - \mathfrak{G}\big(\mathfrak{B}^T,\mathfrak{B}^{(k)T}\big) = \frac{1}{2}\mathrm{tr}\;\big\{\big(\mathfrak{B}-\mathfrak{B}^{(k)}\big)\mathbf{D}\big(\mathfrak{B}-\mathfrak{B}^{(k)}\big)^T\big\}. 
\end{equation*}
By the fact that $\mathbf{D}$ is positive definite for all $\mathfrak{B}\ne\mathfrak{B}^{(k)}$ and positive semi-definite if and only if  $\mathfrak{B}=\mathfrak{B}^{(k)}$, it is proven that $\mathfrak{G}\left(\mathfrak{B}^T,\mathfrak{B}^{(k)T}\right) \le \mathfrak{G}\left(\mathfrak{B}^{(k)T},\mathfrak{B}^{(k)T}\right)$ with the equality happens if and only if $\mathbf{B}^{(k)}$ satisfies the KKT conditions.
\end{proof}

The following theorem summarizes the above results.
\begin{theorem}\label{theorem21}
Given sufficiently large $\delta_{\mathbf{B}}^{(k)}$ and the boundedness of $\mathbf{B}^{(k)}$, $\mathbf{C}^{(k)}$, and $\mathbf{S}^{(k)}$, $J\big(\mathbf{B}^{(k+1)}\big)$ $\le$ $J\big(\mathbf{B}^{(k)}\big)\;\,\forall k\ge 0$ under update rule eq.~\ref{eq98} with the equality holds if and only if $\mathbf{B}^{(k)}$ satisfies the KKT conditions in eqs.~\ref{eq91}. 
\end{theorem}

\subsection{Nonincreasing property of $J\big(\mathbf{C}^{(k)}\big)$}
\begin{theorem}\label{theorem25}
Given sufficiently large $\delta_{\mathbf{C}}^{(k)}$ and the boundedness of $\mathbf{B}^{(k)}$, $\mathbf{C}^{(k)}$, and $\mathbf{S}^{(k)}$, $J\big(\mathbf{C}^{(k+1)}\big)$ $\le$ $J\big(\mathbf{C}^{(k)}\big)\;\,\forall k\ge 0$ under update rule eq.~\ref{eq99} with the equality holds if and only if $\mathbf{C}^{(k)}$ satisfies the KKT conditions in eqs.~\ref{eq91}. 
\end{theorem}
\begin{proof}
This theorem can be proven similarly as in $J\big(\mathbf{B}^{(k)}\big)$ case.
\end{proof}

\subsection{Nonincreasing property of $J\big(\mathbf{S}^{(k)}\big)$}

Next we prove the nonincreasing property of $J\big(\mathbf{S}^{(k)}\big)$, i.e., $J\big(\mathbf{S}^{(k+1)}\big)\le J\big(\mathbf{S}^{(k)}\big)\;\forall k\ge 0$.

By using the auxiliary function approach, the nonincreasing property of $J\big(\mathbf{S}^{(k)}\big)$ can be proven by showing that:
\begin{equation*}
J\big(\mathbf{S}^{(k+1)}\big)= G\big(\mathbf{S}^{(k+1)},\mathbf{S}^{(k+1)}\big) \le G\big(\mathbf{S}^{(k+1)},\mathbf{S}^{(k)}\big) \le G\big(\mathbf{S}^{(k)},\mathbf{S}^{(k)}\big) = J\big(\mathbf{S}^{(k)}\big). 
\end{equation*}
To define $G$, $\mathbf{S}$ is rearranged into:
\begin{equation*}
\mathfrak{S} \equiv
\begin{bmatrix}
\mathbf{s}_1 & & & \\
& \mathbf{s}_2 & & \\
& & \ddots & \\
& & & \mathbf{s}_Q
\end{bmatrix}\in\mathbb{R}_+^{PQ\times Q},
\end{equation*}
where $\mathbf{s}_q$ is the $q$-th column of $\mathbf{S}$. And also let us define:
\begin{equation*}
\nabla_{\mathfrak{S}}\mathfrak{J}\big(\mathfrak{S}^{(k)}\big) \equiv
\begin{bmatrix}
\nabla_{\mathbf{S}}\mathfrak{J}\big(\mathbf{S}^{(k)}\big)_1 & & & \\
& \nabla_{\mathbf{S}}\mathfrak{J}\big(\mathbf{S}^{(k)}\big)_2 & & \\
& & \ddots & \\
& & & \nabla_{\mathbf{S}}\mathfrak{J}\big(\mathbf{S}^{(k)}\big)_Q
\end{bmatrix}\in\mathbb{R}_+^{PQ\times Q},
\end{equation*}
where $\nabla_{\mathbf{S}}\mathfrak{J}\big(\mathbf{S}^{(k)}\big)_q$ is the $q$-th column of $\nabla_{\mathbf{S}}J(\mathbf{S}^{(k)})$. And:
\begin{equation*}
\mathbf{D} \equiv \mathrm{diag}\;\big(\mathbf{D}^1,\ldots,\mathbf{D}^Q\big)\in\mathbb{R}_+^{PQ\times PQ}, 
\end{equation*}
where $\mathbf{D}^q$ is a diagonal matrix with its diagonal entries defined as:
\begin{equation*}
d_{pp}^q \equiv \left\{
 \begin{array}{ll}
   \frac{\big( \mathbf{B}^{(k+1)T}\mathbf{B}^{(k+1)}\mathbf{\bar{S}}^{(k)}\mathbf{C}^{(k+1)}\mathbf{C}^{(k+1)T} \big)_{pq}+\delta_{\mathbf{S}}^{(k)}}{\bar{s}_{pq}^{(k)}} & \mathrm{if}\;\; p\in \mathcal{I}_q \\
   \star & \mathrm{if}\;\; p\notin \mathcal{I}_q
 \end{array} \right.
\end{equation*}
with
\begin{align*}
\mathcal{I}_q \equiv \big\{p|&s_{pq}^{(k)}>0,\;\nabla_{\mathbf{S}}J\big(\mathbf{S}^{(k)}\big)_{pq}\ne 0,\;\mathrm{or} \\
&s_{pq}^{(k)}=0,\;\nabla_{\mathbf{S}}J\big(\mathbf{S}^{(k)}\big)_{pq} < 0\big\}
\end{align*}
is the set of non-KKT indices in $q$-th column of $\mathbf{S}^{(k)}$, and $\star$ is defined as before.

Then, the auxiliary function $\mathfrak{G}$ can be written as:
\begin{equation}
\mathfrak{G}\big(\mathfrak{S},\mathfrak{S}^{(k)}\big) \equiv \;\mathfrak{J}\big(\mathfrak{S}^{(k)}\big) + \mathrm{tr}\;\big\{\big(\mathfrak{S}-\mathfrak{S}^{(k)}\big)^T\nabla_{\mathfrak{S}}\mathfrak{J}\big(\mathfrak{S}^{(k)}\big)\big\} + \frac{1}{2}\mathrm{tr}\;\big\{\big(\mathfrak{S}-\mathfrak{S}^{(k)}\big)^T\mathbf{D}\big(\mathfrak{S}-\mathfrak{S}^{(k)}\big)\big\}. \label{123}
\end{equation}
Also:
\begin{equation*}
\nabla_{\mathfrak{S}}\mathfrak{G}\big(\mathfrak{S},\mathfrak{S}^{(k)}\big)=\mathbf{D}\big(\mathfrak{S}-\mathfrak{S}^{(k)}\big) + \nabla_{\mathfrak{S}}\mathfrak{J}\big(\mathfrak{S}^{(k)}\big).
\end{equation*}
Since $\mathfrak{G}\big(\mathfrak{S},\mathfrak{S}^{(k)}\big)$ is a strict convex function, it has a unique minimum.
\begin{align}
\mathbf{D}\big(\mathfrak{S}-\mathfrak{S}^{(k)}\big) + \nabla_{\mathfrak{S}}\mathfrak{J}\big(\mathfrak{S}^{(k)}\big)=0, \label{124}\\
\mathfrak{S} = \mathfrak{S}^{(k)} - \mathbf{D}^{-1}\nabla_{\mathfrak{S}}\mathfrak{J}\big(\mathfrak{S}^{(k)}\big), \nonumber
\end{align}
which is exactly the update rule for $\mathbf{S}$ in eq.~\ref{eq100}.

By using the Taylor series, an alternative formulation for $\mathfrak{J}\big(\mathfrak{S}\big)$ can be written as:
\begin{equation}
\mathfrak{J}\big(\mathfrak{S}\big) = \;\mathfrak{J}\big(\mathfrak{S}^{(k)}\big) + \mathrm{tr}\;\big\{\big(\mathfrak{S}-\mathfrak{S}^{(k)}\big)^T\nabla_{\mathfrak{S}}\mathfrak{J}\big(\mathfrak{S}^{(k)}\big)\big\} + \frac{1}{2}\mathrm{tr}\;\big\{\big(\mathfrak{S}-\mathfrak{S}^{(k)}\big)^T\nabla_{\mathbf{S}}^2\mathbf{J}\big(\mathbf{S}^{(k)}\big)\big(\mathfrak{S}-\mathfrak{S}^{(k)}\big)\big\} \label{125}
\end{equation}
where
\begin{equation*}
\nabla_{\mathbf{S}}^2\mathbf{J}\big(\mathbf{S}^{(k)}\big) \equiv
\begin{bmatrix}
\nabla_{\mathbf{S}}^2 J \big(\mathbf{S}^{(k)}\big) & & \\
& \ddots & \\
& & \nabla_{\mathbf{S}}^2 J \big(\mathbf{S}^{(k)}\big)
\end{bmatrix}\in\mathbb{R}_+^{PQ\times PQ}
\end{equation*}
with $\nabla_{\mathbf{S}}^2 J \big(\mathbf{S}^{(k)}\big)$ components are arranged along its diagonal area (there are $Q$ components).

For $\mathfrak{G}$ to be the auxiliary function, we must prove:
\begin{enumerate}
\item $\mathfrak{G}\big(\mathfrak{S},\mathfrak{S}\big)=\mathfrak{J}\big(\mathfrak{S}\big)$,
\item $\mathfrak{G}\big(\mathfrak{S}^{(k)},\mathfrak{S}^{(k)}\big)=\mathfrak{J}\big(\mathfrak{S}^{(k)}\big)$,
\item $\mathfrak{G}\big(\mathfrak{S},\mathfrak{S}\big) \le \mathfrak{G}\big(\mathfrak{S},\mathfrak{S}^{(k)}\big)$, and
\item $\mathfrak{G}\big(\mathfrak{S},\mathfrak{S}^{(k)}\big) \le \mathfrak{G}\big(\mathfrak{S}^{(k)},\mathfrak{S}^{(k)}\big)$,
\end{enumerate}
The first and second are clear from the definition of $\mathfrak{G}$ in eq.~\ref{123}, the third and the fourth statements are proven below.


\begin{theorem} \label{theorem27}
Given sufficiently large $\delta_{\mathbf{S}}^{(k)}$ and the boundedness of $\mathbf{B}^{(k)}$, $\mathbf{C}^{(k)}$, and $\mathbf{S}^{(k)}$, then it can be shown that $\mathfrak{G}\big(\mathfrak{S},\mathfrak{S}\big) \le \mathfrak{G}\big(\mathfrak{S},\mathfrak{S}^{(k)}\big)$. Moreover, if and only if $\mathbf{S}^{(k)}$ satisfies the KKT conditions, then the equality holds.
\end{theorem}
\begin{proof}
As $\mathfrak{G}\big(\mathfrak{S},\mathfrak{S}\big)=\mathfrak{J}\big(\mathfrak{S}\big)$, we need to show that $\mathfrak{G}\big(\mathfrak{S},\mathfrak{S}^{(k)}\big)-\mathfrak{J}\big(\mathfrak{S}\big) \ge 0$. By substracting eq.~\ref{123} from eq.~\ref{125}, we get:
\begin{align}
\mathfrak{G}\big(\mathfrak{S},\mathfrak{S}^{(k)}\big)-\mathfrak{J}\big(\mathfrak{S}\big)&=\frac{1}{2}\,\mathrm{tr}\,\big\{\big(\mathfrak{S}-\mathfrak{S}^{(k)}\big)^T\big(\mathbf{D}-\nabla_{\mathbf{S}}^2\mathbf{J}\big(\mathbf{S}^{(k)}\big)\big)\big(\mathfrak{S}-\mathfrak{S}^{(k)}\big)\big\} \nonumber \\
&=\frac{1}{2}\sum_{q=1}^Q\left[\big(\mathbf{s}_q-\mathbf{s}_q^{(k)}\big)^T\big(\mathbf{D}^q-\nabla_{\mathbf{S}}^2 J \big(\mathbf{S}^{(k)}\big)\big)\big(\mathbf{s}_q-\mathbf{s}_q^{(k)}\big)\right]. \label{126}
\end{align}
Let $\mathbf{v}_q = \mathbf{s}_q - \mathbf{s}_q^{(k)}$, then:
\begin{align*}
\mathbf{v}_q^T\big(\mathbf{D}^q-\nabla_{\mathbf{S}}^2J\big(\mathbf{S}^{(k)}\big)\big)\mathbf{v}_q &= \mathbf{v}_q^T\big(\mathbf{D}^q - \big( \mathbf{B}^{(k+1)T}\mathbf{B}^{(k+1)}\mathbf{C}^{(k+1)}\mathbf{C}^{(k+1)T} \big)\big)\mathbf{v}_q \\
&= \mathbf{v}_q^T\big(\mathbf{\bar{D}}^q + \delta_{\mathbf{S}}^{(k)}\mathbf{\hat{D}}^q - \big( \mathbf{B}^{(k+1)T}\mathbf{B}^{(k+1)}\mathbf{C}^{(k+1)}\mathbf{C}^{(k+1)T} \big)\big)\mathbf{v}_q,
\end{align*}
where $\mathbf{\bar{D}}^q$ and $\delta_{\mathbf{S}}^{(k)}\mathbf{\hat{D}}^q$ are diagonal matrices that summed up to $\mathbf{D}^q$, with
\begin{align*}
\bar{d}_{pp}^q &\equiv \left\{
 \begin{array}{ll}
   \frac{\big( \mathbf{B}^{(k+1)T}\mathbf{B}^{(k+1)}\mathbf{\bar{S}}^{(k)}\mathbf{C}^{(k+1)}\mathbf{C}^{(k+1)T} \big)_{pq}}{\bar{s}_{pq}^{(k)}} & \mathrm{if}\;\; p\in \mathcal{I}_q \\
   \star & \mathrm{if}\;\; p\notin \mathcal{I}_q,
 \end{array} \right.
\text{and}\;
\hat{d}_{pp}^q &\equiv \left\{
 \begin{array}{ll}
   \frac{1}{\bar{s}_{pq}^{(k)}} & \mathrm{if}\;\; p\in \mathcal{I}_q \\
   \star & \mathrm{if}\;\; p\notin \mathcal{I}_q.
 \end{array} \right.
\end{align*}
Accordingly,
\begin{align}
\mathfrak{G}\big(\mathfrak{S},\mathfrak{S}^{(k)}\big)-\mathfrak{J}\big(\mathfrak{S}\big) = &\frac{1}{2} \sum_{q=1}^Q \left\{\sum_{p=1}^P v_{pq}^2 \bar{d}_{pp}^q + \delta_{\mathbf{S}}^{(k)} \sum_{p=1}^P v_{pq}^2 \hat{d}_{pp}^q \right\} \nonumber \\
&- \frac{1}{2} \sum_{q=1}^Q \mathbf{v}_q^T \big( \mathbf{B}^{(k+1)T}\mathbf{B}^{(k+1)}\mathbf{C}^{(k+1)}\mathbf{C}^{(k+1)T} \big)\mathbf{v}_q. \label{127}
\end{align}
As shown, with the boundedness of $\mathbf{B}^{(k)}$, $\mathbf{C}^{(k)}$, and $\mathbf{S}^{(k)}$, and by sufficiently large $\delta_{\mathbf{S}}^{(k)}$, $\mathfrak{G}\big(\mathfrak{S},\mathfrak{S}\big) \le \mathfrak{G}\big(\mathfrak{S},\mathfrak{S}^{(k)}\big)$ can be guaranteed. 

Next we prove the second statement of the theorem. By eq.~\ref{126} if $\mathbf{S}^{(k)}$ satisfies the KKT conditions, then the equality will hold. And by eq.~\ref{127}, since $\delta_{\mathbf{S}}^{(k)}$ is a variable, the equality will hold if and only if $\mathbf{S} = \mathbf{S}^{(k)}$ which by the update rule in eq.~\ref{eq100} will happen if and only if $\mathbf{S}^{(k)}$ satisfies the KKT conditions. This completes the proof.
\end{proof}

\begin{theorem} \label{theorem28}
$\mathfrak{G}\big(\mathfrak{S},\mathfrak{S}^{(k)}\big) \le \mathfrak{G}\big(\mathfrak{S}^{(k)},\mathfrak{S}^{(k)}\big)$. Moreover if and only if $\mathbf{S}^{(k)}$ satisfies the KKT conditions, then $\mathfrak{G}\big(\mathfrak{S},\mathfrak{S}^{(k)}\big) = \mathfrak{G}\big(\mathfrak{S}^{(k)},\mathfrak{S}^{(k)}\big)$.
\end{theorem}
\begin{proof}
\begin{equation*}
\mathfrak{G}\big(\mathfrak{S}^{(k)},\mathfrak{S}^{(k)}\big) - \mathfrak{G}\big(\mathfrak{S},\mathfrak{S}^{(k)}\big) = -\mathrm{tr}\;\big\{\big(\mathfrak{S}-\mathfrak{S}^{(k)}\big)^T\nabla_{\mathfrak{S}}\mathfrak{J}\big(\mathfrak{S}^{(k)T}\big)\big\} - \frac{1}{2}\mathrm{tr}\;\big\{\big(\mathfrak{S}-\mathfrak{S}^{(k)}\big)^T\mathbf{D}\big(\mathfrak{S}-\mathfrak{S}^{(k)}\big)\big\}.
\end{equation*}
Substituting eq.~\ref{124} into the above equation, we get:
\begin{equation*}
\mathfrak{G}\big(\mathfrak{S}^{(k)},\mathfrak{S}^{(k)}\big) - \mathfrak{G}\big(\mathfrak{S},\mathfrak{S}^{(k)}\big) = \frac{1}{2}\mathrm{tr}\;\big\{\big(\mathfrak{S}-\mathfrak{S}^{(k)}\big)^T\mathbf{D}\big(\mathfrak{S}-\mathfrak{S}^{(k)}\big)\big\} \ge 0, 
\end{equation*}
By the fact that $\mathbf{D}$ is positive definite for all $\mathfrak{S}\ne\mathfrak{S}^{(k)}$ and positive semi-definite if and only if  $\mathfrak{S}=\mathfrak{S}^{(k)}$, it is proven that $\mathfrak{G}\left(\mathfrak{S},\mathfrak{S}^{(k)}\right) \le \mathfrak{G}\left(\mathfrak{S}^{(k)},\mathfrak{S}^{(k)}\right)$ with the equality holds if and only if $\mathbf{S}^{(k)}$ satisfies the KKT conditions.
\end{proof}

The following theorem summarizes the above results.
\begin{theorem}\label{theorem29}
Given sufficiently large $\delta_{\mathbf{S}}^{(k)}$ and the boundedness of $\mathbf{B}^{(k)}$, $\mathbf{C}^{(k)}$, and $\mathbf{S}^{(k)}$, $J\big(\mathbf{S}^{k+1}\big)$ $\le$ $J\big(\mathbf{S}^{(k)}\big)\;\,\forall k\ge 0$ under update rule eq.~\ref{eq100} with the equality holds if and only if $\mathbf{S}^{(k)}$ satisfies the KKT conditions in eq.~\ref{eq91}. 
\end{theorem}

\subsection{The nonincreasing property of sequence $J\big(\mathbf{B}^{(k)}, \mathbf{S}^{(k)}, \mathbf{C}^{(k)}\big)$}
As stated in the beginning of the appendix, this property is the first point needs to be proven in order to show the convergence of algorithm \ref{algorithm8}.
\begin{theorem} \label{theorem30}
Given sufficiently large $\delta_{\mathbf{B}}^{(k)}$, $\delta_{\mathbf{C}}^{(k)}$, and $\delta_{\mathbf{S}}^{(k)}$, and the boundedness of $\mathbf{B}^{(k)}$, $\mathbf{C}^{(k)}$, and $\mathbf{S}^{(k)}$, $J\big(\mathbf{B}^{(k+1)}$, $\mathbf{S}^{(k+1)}$, $\mathbf{C}^{(k+1)}\big)$ $\le$ $J\big(\mathbf{B}^{(k+1)}$, $\mathbf{S}^{(k)}$, $\mathbf{C}^{(k+1)}\big)$ $\le$ $J\big(\mathbf{B}^{(k+1)}$, $\mathbf{S}^{(k)}$, $\mathbf{C}^{(k)}\big)$ $\le$ $J\big(\mathbf{B}^{(k)}$, $\mathbf{S}^{(k)}$, $\mathbf{C}^{(k)}\big)$ for $\forall k\ge 0$ under update rules in algorithm \ref{algorithm8} with the equalities happen if and only if $\big(\mathbf{B}^{(k)}$, $\mathbf{S}^{(k)}$, $\mathbf{C}^{(k)}\big)$ satisfies the KKT optimality conditions.
\end{theorem}
\begin{proof}
This statement can be proven by combining the results in theorems \ref{theorem21}, \ref{theorem25}, and \ref{theorem29} .
\end{proof}

\subsection{Limit points of sequence $\left\{\mathbf{B}^{(k)}, \mathbf{S}^{(k)}, \mathbf{C}^{(k)}\right\}$}
\begin{theorem} \label{theorem31}
Given sufficiently large $\delta_{\mathbf{B}}^{(k)}$, $\delta_{\mathbf{C}}^{(k)}$, and $\delta_{\mathbf{S}}^{(k)}$, and the boundedness of $\mathbf{B}^{(k)}$, $\mathbf{C}^{(k)}$, and $\mathbf{S}^{(k)}$, it can be shown that any limit point of sequence $\big\{\mathbf{B}^{(k)},\mathbf{S}^{(k)},\mathbf{C}^{(k)}\big\}$ generated by algorithm \ref{algorithm8} is a stationary point.
\end{theorem}
\begin{proof}
By theorem \ref{theorem30}, algorithm \ref{algorithm8} produces strictly decreasing sequence $J\big(\mathbf{B}^{(k)}$, $\mathbf{S}^{(k)}$, $\mathbf{C}^{(k)}\big)$ until reaching a point that satisfies the KKT conditions. Because $J\big(\mathbf{B}^{(k)}$, $\mathbf{S}^{(k)}$, $\mathbf{C}^{(k)}\big)$ $\ge$ $0$, this sequence is bounded and thus converges. And by the update rules in algorithm \ref{algorithm8}, after a point satisfies the KKT conditions, the algorithm will stop updating $\big(\mathbf{B}^{(k)}$, $\mathbf{S}^{(k)}$, $\mathbf{C}^{(k)}\big)$, i.e., $\mathbf{B}^{(k+1)}$ $=$ $\mathbf{B}^{(k)}$, $\mathbf{C}^{(k+1)}$ $=$ $\mathbf{C}^{(k)}$, and $\mathbf{S}^{(k+1)}$ $=$ $\mathbf{S}^{(k)}\;\,\forall k\ge *$ where $*$ denotes the iteration number where the KKT conditions have been satisfied. This completes the proof.
\end{proof}

\begin{theorem} \label{theorem32}
The sequence $\big\{\mathbf{B}^{(k)},\mathbf{S}^{(k)},\mathbf{C}^{(k)}\big\}$ has at least one limit point.
\end{theorem}
\begin{proof}
Due to the result in theorem \ref{theorem31}, it suffices to prove that sequence $\big\{\mathbf{B}^{(k)},\mathbf{S}^{(k)},\mathbf{C}^{(k)}\big\}$ is in a closed and bounded set. By the objective in eq.~\ref{eq90} it is clear that $\left\{\mathbf{B}^{(k)}, \mathbf{S}^{(k)},\mathbf{C}^{(k)}\right\}$ has a nonnegative lower bound. Thus, only upper-boundedness of this sequence needs to be proven. If there exists $l$ such that $\lim b_{mp}^{(l)}\to \infty$ or $\lim c_{qn}^{(l)}\to \infty$, then $\lim J \to \infty > J(\mathbf{B}^{(0)},\mathbf{S}^{(0)},\mathbf{C}^{(0)})$ which violates theorem \ref{theorem30}\footnote{If $\beta$ or $\alpha$ in eq.~\ref{eq90} is set to zero, then it is possible to have a bounded $J$; however in this case the problem will be reduced to uni-orthogonal NMF where we can use the proof of theorem 13 in \cite{Mirzal2} to prove this theorem.}. And if $\big\{\mathbf{S}^{(k)}\big\}$ is not upper-bounded, then there exists $l$ such that $\lim s_{pq}^{(l)}\to \infty$, $s_{pq}^{(l)} < s_{pq}^{(l+1)}$. Due to theorem \ref{theorem30} $J(\mathbf{B}^{(k)},\mathbf{S}^{(k)},\mathbf{C}^{(k)})$ is upper-bounded, then this means that either $b_{mp}^{(l)}$ for $\forall m$ or $c_{qn}^{(l)}$ for $\forall n$ must be equal to zero. If $b_{mp}^{(l)}=0$ for $\forall m$, then $\nabla_{\mathbf{S}}J_{pq}=0$ for $\forall q$, so that $s_{pq}^{(l+1)} = s_{pq}^{(l)}$. And if $c_{qn}^{(l)}=0$ for $\forall n$, then $\nabla_{\mathbf{S}}J_{pq}=0$ for $\forall p$, so that $s_{pq}^{(l+1)} = s_{pq}^{(l)}$. Both cases contradict the condition for unboundedness of $\big\{\mathbf{S}^{(l)}\big\}$. Thus, $\big\{\mathbf{S}^{(l)}\big\}$ is also upper-bounded.
\end{proof}

\end{document}